\documentclass[letterpaper]{article} 
\usepackage{aaai2026}  
\nocopyright
\usepackage{times}  
\usepackage{helvet}  
\usepackage{courier}  
\usepackage[hyphens]{url}  
\usepackage{graphicx} 
\usepackage{natbib}  
\usepackage{caption} 
\usepackage[utf8]{inputenc}
\usepackage{booktabs}
\usepackage{multirow}
\usepackage{bm}
\usepackage{amsthm}
\usepackage{amsmath}
\usepackage{amssymb}
\usepackage{tabularx}
\usepackage{color}
\usepackage[table]{xcolor}
\usepackage{algorithm}
\usepackage{algorithmic}

\usepackage{newfloat}
\usepackage{listings}

\newtheorem{Theorem}{\bf Theorem}
\newtheorem{Definition}{\bf Definition}

\newtheorem{Corollary}{\bf Corollary}
\newtheorem{Assumption}{\bf Assumption}
\newcommand{\Perp}{\perp\!\!\!\perp}

\newcommand{\multirowoffset}{-0.5\dimexpr \aboverulesep + \belowrulesep + \cmidrulewidth}

\DeclareCaptionStyle{ruled}{labelfont=normalfont,labelsep=colon,strut=off} 
\floatstyle{ruled}
\newfloat{listing}{tb}{lst}{}
\floatname{listing}{Listing}
\title{Robust Causal Discovery under Imperfect Structural Constraints}
\author{
    Zidong Wang\textsuperscript{\rm 1},
    Xi Lin\textsuperscript{\rm 1},
    Chuchao He\textsuperscript{\rm 2,}\thanks{Corresponding author.},
    Xiaoguang Gao\textsuperscript{\rm 3}
}
\affiliations{
    \textsuperscript{\rm 1} Department of Computer Science, City University of Hong Kong, Hong Kong, China\\
    \textsuperscript{\rm 2} School of Electronic Information Engineering, Xi'an Technological University, Xi'an, China\\
    \textsuperscript{\rm 3} School of Electronics And Information, Northwestern Polytechnical University, Xi'an, China\\

    \{zidowang, xilin4\}@cityu.edu.hk, hechuchao@xatu.edu.cn, cxg2012@nwpu.edu.cn\\

}

\usepackage{bibentry}

\begin{document}

\maketitle

\begin{abstract}
Robust causal discovery from observational data under imperfect prior knowledge remains a significant and largely unresolved challenge. Existing methods typically presuppose perfect priors or can only handle specific, pre-identified error types. And their performance degrades substantially when confronted with flawed constraints of unknown location and type. This decline arises because most of them rely on inflexible and biased thresholding strategies that may conflict with the data distribution. To overcome these limitations, we propose to harmonizes knowledge and data through prior alignment and conflict resolution. First, we assess the credibility of imperfect structural constraints through a surrogate model, which then guides a sparse penalization term measuring the loss between the learned and constrained adjacency matrices.  We theoretically prove that, under ideal assumption, the knowledge-driven objective aligns with the data-driven objective. Furthermore, to resolve conflicts when this assumption is violated, we introduce a multi-task learning framework optimized via multi-gradient descent, jointly minimizing both objectives. Our proposed method is robust to both linear and nonlinear settings. Extensive experiments, conducted under diverse noise conditions and structural equation model types, demonstrate the effectiveness and efficiency of our method  under imperfect structural constraints.
\end{abstract}

%
\begin{links}
    \link{Code}{https://github.com/wzd2502/RoaDs}
\end{links}

\section{Introduction}
Causal discovery from observational data is a cornerstone of artificial intelligence and scientific inquiry \cite{DBLP:books/daglib/0023012, pearl2009causality}. By revealing the underlying causal mechanism and representing as a directed acyclic graph (DAG), it provides the fundamental structure required for downstream tasks such as causal inference \cite{hernan2010causal, peters2017elements}, and causal representation learning \cite{scholkopf2021toward, brehmer2022weakly}. A central topic in causal discovery is identifiability \cite{vowels2022d}. Under the causal sufficiency and faithfulness assumptions \cite{koller2009probabilistic}, traditional combinatorial optimization methods can identify the structure up to its Markov Equivalence Class (MEC), which is also known as Bayesian network structure learning \cite{glymour2019review, kitson2023survey}. This full DAG-level identifiability can be achieved either by using interventional data or by imposing stricter assumptions on the data-generating process, such as non-Gaussian noise or nonlinear structural equation models (SEMs) \cite{vowels2022d}. These stronger assumptions often enable the problem to be cast as a continuous optimization problem, making it solvable by zero-order \cite{shimizu2011directlingam}, first-order \cite{zheng2018dags, ng2020role}, or second-order optimization methods \cite{rolland2022score}.

\begin{figure}
    \centering
    \includegraphics[width=1\linewidth]{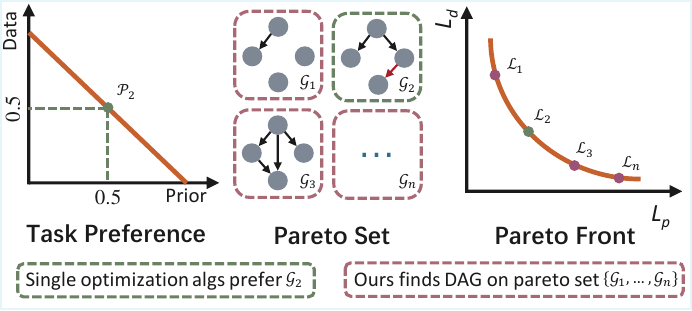}
    \caption{Robustness to imperfect constraints. A single-objective baseline assigns equal weight and is misled by the flawed prior (red arrow in $\mathcal{G}_2$), whereas ours identifies the conflict and discovers DAG on the Pareto set.}
    \label{fig:motivation}
\end{figure}
However, in numerous real-world applications, such as rare disease diagnosis or industrial fault analysis, high-quality observational data are often scarce and difficult to obtain. These domains typically possess a wealth of expert prior knowledge (e.g., positive or negative edge constraints) \cite{constantinou2023impact, brouillard2024landscape}. Consequently, how to effectively integrate such prior knowledge with data-driven methods has become an important yet challenging research direction.

Most existing methods are designed for perfect priors (no errors in constraints): combinatorial-based approaches typically treat priors as hard constraints, such as initializing the search or populating a tabu list \cite{de2007bayesian, chen2025mitigating, wang2025uncertain}; continuous-based approaches incorporate priors as soft penalty terms or as hard optimization goals \cite{sun2023nts, chen2025continuous}. In practice, however, expert knowledge is often imperfect, potentially containing overlooked true causal edges or erroneously introduced spurious ones. When faced with such imperfect priors, the performance of existing methods degrades sharply. 
We explicate this issue from a multi-objective optimization perspective in Figure ~\ref{fig:motivation}. Previous works typically use weighted sum scalarization to combine the data-driven and knowledge-driven objectives, which restricts the solution to a single, predetermined point on the Pareto front. Furthermore, these methods can neither adaptively correct erroneous priors nor adjust the weight of the knowledge-based objective to reflect its credibility. When priors are unreliable, a fixed, high weight inevitably forces the model to overfit to this incorrect DAG, such as \(\mathcal{G}_2\) in Pareto set.


To tackle this dilemma, we build upon continuous optimization methods to develop a robust framework capable of handling imperfect structural constraints. Our approach achieves this through two components: \textbf{Prior Alignment}, which employs a surrogate model dynamically modulating the weights of imperfect constraints based on the observational data; \textbf{Conflict Resolution}, which leverages multi-task learning (MTL) to explicitly manage the trade-off between the data-driven and knowledge-driven objectives. We named it as \underline{Ro}bust C\underline{a}usal \underline{D}iscovery under Imperfect \underline{s}tructural constraints (\textbf{RoaDs}). Our main contributions are as follows:

\begin{itemize}
    \item We introduce a consistent constraint assumption and use a surrogate model to learn continuous weights for priors.  
    
    \item We design the knowledge-based optimization goal based on consistent constraints, and theoretically prove the asymptotic consistency of it.
 
    \item We employ Multi Gradient Descent Algorithm (MGDA), enhanced with gradient normalization, to efficiently find a balanced Pareto stationary point for MTL problem.

    \item In experimental evaluation, we demonstrate the superior robustness and effectiveness of RoaDs against SOTA methods across diverse and challenging settings.
\end{itemize}

\section{Related Works}

\paragraph{Causal discovery under structural constraints} 
For combinatorial-based methods, integrating edge constraints is relatively straightforward, typically by restricting the search space \cite{de2007bayesian, colombo2014order,constantinou2023impact}. However, path constraints, which are weaker and non-decomposable, need the graphical search space or specialized data structures to entail \cite{chen2016learning, wang2021learning, wang2025large}. A key limitation of these approaches is their reliance on the assumption that all provided constraints are perfect and error-free. For continuous-based approaches, perfect edge constraints are often handled in two ways: either enforced as hard constraints that are optimized simultaneously with the acyclicity constraint \cite{hasan2022kcrl,sun2023nts,wang2024incorporating}, or by directly modifying the gradients of the adjacency matrix to steer the search \cite{bello2022dagma}. Imperfect priors are typically handled via soft penalties, where constraints are formulated as differentiable terms, such as a cross-entropy loss measuring constraint violation \cite{li2024weakly, chen2025continuous}. To handle path constraints, this paradigm involves employing partial order-based optimization strategies \cite{bandifferentiable}.

More recently, a nascent line of work has explored using Large Language Models (LLMs) as a proxy for domain experts \cite{kiciman2023causal}. LLMs have been used to generate initial graphs \cite{ban2025integrating}, suggest post-hoc adjustments \cite{khatibi2024alcm}, or fuse structural priors from text \cite{zhou2024causalbench, ban2025llm}.

For a broader survey of general causal discovery methods, we refer the reader to Appendix A.

\paragraph{Multi-task Learning} MTL is quite a hot topic in the machine learning community \cite{zhang2021survey}. MTL can improve the generalization and reduce the cost of learned models, thus it is widely applied in many scenarios \cite{zhao2022inherent}. Key research in MTL involves designing shared architectures and managing conflicting task objectives \cite{lin2023libmtl}. Our work concentrates on the latter, employing multi-objective optimization (MOO) to mitigate the conflict between data-driven and knowledge-driven objectives for causal discovery.

MOO solvers can be broadly categorized into two families \cite{zhang2024libmoon}. The first, aggregation-based methods, transforms the multi-objective problem into a single-objective one by aggregating individual loss functions, such as Linear scalarization \cite{miettinen1999nonlinear}, the Tchebycheff method \cite{zhang2007moea}, Smooth TCH \cite{lin2024smooth}. The second family, gradient-manipulation-based methods, operates directly on the gradients of each task to find a descent direction that improves all objectives. Prominent examples include the MGDA \cite{sener2018multi}, its preference-based extensions \cite{lin2019pareto}, and normalization version \cite{chen2018gradnorm}.

\section{Preliminary}

\subsection{Causal discovery}
A causal structure can be represented by a DAG \( \mathcal{G} = (\bm{V}, \bm{E}) \), where \( \bm{V} = \{X_1, \dots, X_{n_v}\} \) is a set of variables and \( \bm{E} \) is the set of edges. An edge \( X_i \to X_j \) implies that \( X_i \) is a direct cause (parent) of \( X_j \) \cite{koller2009probabilistic}, denoted as \( X_i \in \Pi_j^{\mathcal{G}} \). We consider the Additive Noise Model (ANM) \cite{hoyer2008nonlinear}, where each variable is generated by a function of its parents plus an independent noise term
\(X_j = f_j(\Pi_j^{\mathcal{G}}) + \epsilon_j.\)
Here, \( f_j \) is a causal mechanism, and the noise terms \( \bm{\epsilon} = \{\epsilon_1, \dots, \epsilon_{n_v}\} \) are assumed to be mutually independent with zero mean (\(\mathbb{E}[\epsilon_j] = 0\)) and covariance matrix \(\mathrm{diag}({\sigma _1},\dots,{\sigma _{{n_v}}})\). Given an i.i.d. dataset \( \mathbf{X} = [\mathbf{x}_1 | \dots | \mathbf{x}_{n_v}] \in \mathbb{R}^{n_d \times n_v} \), the goal of causal discovery is to find the optimal DAG \( \mathcal{G} \) by solving a continuous optimization problem:
\begin{equation}\label{eq:eq1}
    \begin{array}{l}
    \mathop {\min }\limits_\mathbf{f} \: \sum_{j=1}^{n_v} \mathcal{L}(\mathbf{x}_j, f_j(\mathbf{X})) \\
    s.t. \:  \mathcal{G}(\mathbf{f}) \text{ is acyclic},
    \end{array}
\end{equation}
where \( \mathcal{L}(\cdot) \) is a least squares loss or negative log-likelihood loss, and \( \mathcal{G}(\mathbf{f}) \) is the DAG induced by the functional dependencies in \( \mathbf{f} = \{f_1, \dots, f_{n_v}\} \). Each \( f_j \) can be parameterized using a Multilayer Perceptron: \( f_j(\mathbf{X}) = \mathrm{MLP}(\mathbf{X}; \theta_j) \), where \( \bm{\theta} = \{\theta_1, \dots, \theta_{n_v}\} \). \( \theta_j = \{A_j^{(k)}\}_{k=1}^{n_h} \) are the parameters for the \(j\)-th MLP, and \( A_j^{(k)} \in \mathbb{R}^{d_{k-1} \times d_k} \) denotes the weights of the \(k\)-th layer \cite{lachapellegradient, zheng2020learning}. Under such condition, the weighted adjacency matrix can be approximately expressed as \( W(\bm{\theta}) \in \mathbb{R}^{n_v \times n_v} \). The entry \( [W(\bm{\theta})]_{ij} \) quantifies the causal influence from \( X_i \) to \( X_j \) and is defined as
\([W(\bm{\theta})]_{ij} = \| [A_j^{(1)}]_{:,i} \|_2,\)
Consequently, the optimization problem from Eq.~\eqref{eq:eq1} is reformulated as:
\begin{equation}\label{eq:eq2}
\begin{array}{l}
\mathop {\min }\limits_{\bm{\theta}} \: \frac{1}{n_d} \sum_{j=1}^{n_v} \| \mathbf{x}_j - \mathrm{MLP}(\mathbf{X}; \theta_j) \|_F^2 + \lambda_1 \| W(\bm{\theta}) \|_1 \\
s.t. \: h(W(\bm{\theta})) = \mathrm{tr}(e^{W(\bm{\theta}) \circ W(\bm{\theta})}) - n_v = 0.
\end{array}
\end{equation}
Problem \eqref{eq:eq2} can be transformed into unconstrained optimization form using Augmented Lagrangian Method (ALM) \cite{zheng2018dags}. For brevity, we will henceforth denote the original objective function as \ \( \mathcal{F}_{\mathbf{X}}(\bm{\theta}) \), respectively, and the constrained objective function as  \(\mathcal{H}(W(\bm{\theta}))\)
\begin{equation}\label{eq:eq3}
    \mathcal{H}(W(\bm{\theta})) = \varphi h(W(\bm{\theta})) + {\textstyle{\rho  \over 2}} |h(W(\bm{\theta}))|^2,
\end{equation}
where \(\varphi\) and \(\rho\) are parameters in ALM.

\subsection{Multi-task learning}
A MTL problem can be formulated as a multi-objective optimization problem, where the goal is to simultaneously minimize a vector of loss functions corresponding to different tasks \cite{ caruana1993multitask,miettinen1999nonlinear}:
\begin{equation}\label{eq:eq4}
    \min_{\bm{\theta\in \Theta}} \mathbf{L}(\bm{\theta}) = (\mathcal{L}_1(\bm{\theta}), \dots, \mathcal{L}_{n_p}(\bm{\theta}))^T,
\end{equation}
A solution \( \bm{\theta}_a \) is said to dominate \( \bm{\theta}_b \), denoted as \( \mathbf{L}(\bm{\theta}_a) \prec \mathbf{L}(\bm{\theta}_b) \), if \( \mathcal{L}_k(\bm{\theta}_a) \le \mathcal{L}_k(\bm{\theta}_b) \) holds \(\forall k \in \{1, \dots, n_p\} \), and there exists at least one index \( j \) for which \( \mathcal{L}_j(\bm{\theta}_a) < \mathcal{L}_j(\bm{\theta}_b) \).
\begin{Definition}
    (Pareto Optimality) A solution \( \bm{\theta}^* \in \Theta \) is Pareto optimal if no other solution \( \bm{\theta} \in \Theta \) dominates it, i.e., there is no \( \bm{\theta} \) such that \( \mathbf{L}(\bm{\theta}) \prec \mathbf{L}(\bm{\theta}^*) \).
\end{Definition}
For MTL with conflicting objectives, there not exists a single solution that minimizes all task losses simultaneously. Instead, a set of trade-off solutions exists. The set of all Pareto optimal solutions is called the \emph{Pareto set}, and its image in the objective space is the \emph{Pareto front}.

\section{Framework }

\begin{figure}
    \centering
    \includegraphics[width=1.0\linewidth]{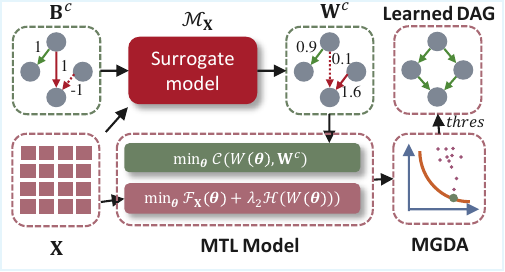}
    \caption{Pipeline. RoaDs constructs a data-driven objective from a continuous score and a knowledge-driven objective using a surrogate model to align imperfect constraints (Red arrows in figure). These are formulated as a MTL problem, which is then solved via the MGDA to recover the final causal graph. }
    \label{fig:pipeline}
\end{figure}

This paper focuses on causal discovery where the available prior knowledge may conflict with the ground-truth graph. And such knowledge can be formally defined as follow.

\begin{Definition}
    (Imperfect constraints.) Let the constraints be encoded in a matrix \( \mathbf{B}^c \in \{0, 1, -1\}^{n_v \times n_v} \), where \( \mathbf{B}^c_{ij} = 1,-1,0 \) signifies positive constraint (\(X_i \to X_j\)), negative constraint (\(X_i \not \to X_j\)), and no constraint. Let \( \mathbf{B}^* \) be the adjacency matrix of the ground-truth graph. \( \mathbf{B}^c \) is considered \textbf{imperfect} if there exist entries \( (i, j) \) such that \( \mathbf{B}^c_{ij} = 1 \) but \( \mathbf{B}^*_{ij} = 0 \), or \( \mathbf{B}^c_{ij} = -1 \) but \( \mathbf{B}^*_{ij} = 1 \).
\end{Definition}

The propose RoaDs refines the imperfect constraints by aligning them with the observational data to against the terrible influence from flawed priors, and resolves the remain conflict between the data-driven and knowledge-driven objectives using a MOO solver, as illustrated in Figure ~\ref{fig:pipeline}.

\subsection{Prior alignment}
The reliability of the prior alignment is fundamentally compromised by the highly non-convex optimization landscape of continuous-based methods. We therefore lay the foundation for RoaDs by first defining a theoretical criterion that acts as the \textit{tool} for overriding flawed priors and simultaneously establishes the \textit{bounds} of its valid application. Note that the subsequent analysis still holds under causal faithfulness and sufficiency assumption.

\paragraph{Tool for alignment.} To enable a uniform re-evaluation of all priors, the negative constraints are firstly converted into positive constraints. Then, a surrogate model \(\bm{\mathcal{M}}=\{\mathcal{M}^{(1)},\dots, \mathcal{M}^{(n_v)}\}\) is employed to test the credibility of different constraints under the dataset \(\mathbf{X}\). For arbitrary \(X_j\), \(\mathcal{M}^{(j)}\) solves \( \mathbb{E}[{X_j}|\Pi _j^{\mathbf{B}}] \) to find the weights of edges point to \(X_j\), which is defined as
\begin{equation}\label{eq:eq5}
    \mathbf{W}_{:,j} = \mathcal{M}^{(j)}_{\mathbf{X}}([\mathbf{B}_{:,j}]_{\neq 0}).
\end{equation}
Thus, when \(n_d \to \infty\), the ground-truth DAG satisfies
\begin{equation}\label{eq:eq6}
    \mathbf{W}^*_{:,j} = \mathcal{M}^{(j)}_{\mathbf{X}}([\mathbf{B}^*_{:,j}]_{\neq 0})=\mathcal{M}^{(j)}_{\mathbf{X}}(\bm{1}_{:,j}).
\end{equation}
For linear case, the surrogate model can be achieved by the consistent parametric regressor, where \(\mathbf{W}_{ij}\) can be represented by the regression coefficients from \(\mathbf{x}_i\) to \(\mathbf{x}_j\). In nonlinear settings, the consistent non-parametric regressor (e,g, random forest) is feasible, and \(\mathbf{W}_{ij}\) can be represented by permutation importance \cite{hastie2009elements}. 

\paragraph{Bounds of alignment.} We introduce a strict assumption about the dependency relations in constraints matrix, that determines whether flawed prior can be aligned.
\begin{Assumption}
    (Consistent constraints.) For the ground-truth DAG \(\mathbf{B^*}\), constraints matrix \(\mathbf{B}^{c}\) is consistent if it satisfies \( \forall X_i \in \bm{V}^{1,0}_j, X_i \Perp X_j| \bm{V}^{1,1}_j  \) and \( \forall X_i \in \bm{V}^{0,1}_j, X_k \in \bm{V}^{1,1}_j, X_i \Perp X_k \), where $\bm{V}^{\alpha,\beta}_j = \{X_k \mid \mathbf{B}^c_{kj}=\alpha, \mathbf{B}^*_{kj}=\beta \}, \: \alpha, \beta \in \{0,1\}$.
\end{Assumption}

\begin{Theorem}\label{thm:thm1}
If \(\mathbf{B}^c\) is consistent, there always exists \(\tau > 0\) such that the probability limit of \(\mathbf{W}^c\) from Eq.~\eqref{eq:eq5} satisfies:
\begin{enumerate}
    \item \(\forall X_i \in \bm{V}^{1,0}_j\), then \(\text{plim}_{n_d \to \infty} \mathbf{W}^c_{ij} < \tau\).
    \item \(\forall X_i \in \bm{V}^{1,1}_j\), then \(\text{plim}_{n_d \to \infty} \mathbf{W}^c_{ij} > \tau\).
\end{enumerate}
\end{Theorem}
The detailed proof is provided in Appendix B. Theorem \ref{thm:thm1} demonstrates that if the constraint matrix \(\mathbf{B}^c\) is consistent, the surrogate model \(\bm{\mathcal{M}}\) successfully recovers true edges while simultaneously rejecting the false positive edges that were incorrectly specified in the \(\mathbf{B}^c\). The resulting weight matrix \(\mathbf{W}^c\) from prior alignment will accurately reflect the \textit{partial} ground-truth structure \( \mathbf{B}^*\). And the following discuss in this section is all based on consistent \(\mathbf{B}^c\).

\paragraph{Knowledge-driven optimization objective.}
After prior alignment, \(\mathbf{W}^c =\mathcal{M}_{\mathbf{X}}([\mathbf{B}^c]_{\neq 0})\) can serve for the modeling of knowledge-driven optimization objective. This objective aims to promote the non-parametric weighted adjacency matrix \(W(\bm{\theta})\) towards to the refined DAG encoded in \(\mathbf{W}^c\). 
Intuitively, this purpose can be achieved by minimizing the \(\ell1\) norm of the difference between their binarized structures, an objective that exclusively evaluates discrepancies at the locations specified by the original constraint mask \(\mathbf{B}^c\)
\begin{equation}\label{eq:eq7}
     \min_{\bm{\theta}} \left\| \left[ \mathbb{I}( {W}(\bm{\theta}) - s > 0) - \mathbb{I}(\mathbf{W}^c - \tau > 0) \right] \circ \mathbf{B}^c \right\|_1.
\end{equation}
\(\mathbb{I}(\cdot)\) denotes the Heaviside step function, which maps its input to \(\{0, 1\}\) based on the specified thresholds \(s\) and \(\tau\). The \(\circ \mathbf{B}^c\) localizes the penalty to the constrained entries. However, the discontinuous nature of \(\mathbb{I}(\cdot)\) renders this objective non-differentiable and thus unamenable to standard gradient-based optimization methods.

To facilitate tractable optimization, we introduce a sub-differentiable form of Eq. \eqref{eq:eq7} by substituting the \(\mathbb{I}(\cdot)\) with a continuous sigmoid function \(\sigma(\cdot)\), which acts as a smooth approximation. This yields the following objective
\begin{equation}\label{eq:eq8}
     \min_{\bm{\theta}} \left\| \left[ \sigma( {W}(\bm{\theta}) - s) - \sigma(\mathbf{W}^c - \tau) \right] \circ \mathbf{B}^c \right\|_1,
\end{equation}
and we denote it as \(\mathcal{C}(W(\bm{\theta}), \mathbf{W}^c)\). For the linear case, Eq.~\eqref{eq:eq8} reduces to a more concise form where \(s = \tau\) and the sigmoid function \(\sigma(\cdot)\) is omitted in favor of a parametric regressor.
We can theoretically show that this formulation achieves a lower error bound than fixed thresholding methods, more detailed is provided in Appendix C.

\paragraph{Asymptotic consistency.}
The following theorem establishes that under a single optimization architecture, which integrates the data-driven optimization objective \(\mathcal{F}_\mathbf{X}(\bm{\theta}) + \lambda_2\mathcal{H}(W(\bm{\theta}))\), and our knowledge-regularization term \(\mathcal{C}\), is asymptotically consistent.

\begin{Theorem}\label{thm:thm2}
Consider the continuous optimization problem defined as:
\begin{equation}\label{eq:eq9}
    \min_{\bm{\theta}} \mathcal{F}_\mathbf{X}(\bm{\theta}) + \lambda_2 \mathcal{H}(W(\bm{\theta})) + \lambda_3 \mathcal{C}(W(\bm{\theta}), \mathbf{W}^c).
\end{equation}
Let \(\hat{\bm{\theta}}\) be the optimal solution to the above problem. As the number of samples \( n_d \to \infty \), the graph structure induced by \({W}(\hat{\bm{\theta}})\) converges in probability to the ground-truth DAG \(\mathbf{B}^*\)
\begin{equation}\label{eq:eq10}
    \mathbb{I}({W}(\hat{\bm{\theta}}) > s) \xrightarrow{p} \mathbf{B}^*.
\end{equation}
\end{Theorem}
The detailed proof is provided in Appendix B.

\paragraph{Dilemma under non-consistent constraints. }
According to Theorem \ref{thm:thm2}, if the imperfect constraints $\mathbf{B}^c$ are consistent, the knowledge-driven objective aligns with the data-driven objective in large-sample settings.
However, a significant gap exists between this asymptotic ideal and practical application. 
First, verifying the consistency of given constraints is often intractable, as it would require a relatively accurate understanding of the ground-truth structure $\mathbf{B}^*$. 
Second, the introduction of prior knowledge is to improve the learning accuracy under small sample size, where theoretical guarantees are weakest.

Consequently, the data-driven term \(\mathcal{F}_{\mathbf{X}}(\bm{\theta})+\lambda_2 \mathcal{H}(W(\bm{\theta}))\) and the knowledge-regularization term \(\mathcal{C}(W(\bm{\theta}),\mathbf{W}^c)\) often remain in conflict, further contributing to a highly non-convex optimization landscape \cite{reisach2021beware,  ng2024structure}. This inherent tension necessitates a more sophisticated mechanism to mediate between data and imperfect constraints.

\subsection{Conflict resolution}
We propose a MTL framework designed to balance these two conflicting objectives. Formally, the two optimization tasks are defined as
\begin{equation}\label{eq:eq11}
    \left\{ {\begin{array}{ll}
     & {\mathop {\min }\limits_\mathbf{\bm{\theta}} \mathcal{F}_\mathbf{X}(\mathbf{\bm{\theta}}) +\lambda_2 \mathcal{H}({W(\bm{\theta})}) }\\
     & {\mathop {\min }\limits_\mathbf{\bm{\theta}}\mathcal{C}(W(\bm{\theta}),\mathbf{W}^c).}
    \end{array}} \right.
\end{equation}
Here we assign the equal preference to both tasks, thus, the parameter \(\lambda_3\) for the second task is omitted.

\paragraph{Solve the MTL problem.}
We employ the MGDA to solve MOO problem in Eq.~\eqref{eq:eq11} \cite{sener2018multi}, as it efficiently identifies a single Pareto-stationary point, instead of the entire Pareto front, which is not friendly to decision-makers. Another advantage is that it can adaptively adjust the weights of the two optimization goals, which is crucial for navigating the conflict between data-driven evidence and imperfect constraints. \(\bm{\theta}\) is updated according to
\begin{equation}\label{eq:eq12}
    \bm{\theta}_{t+1} = \bm{\theta}_{t} + \eta \bm{d}_{t},
\end{equation}
where \(\eta \) is the learning rate, and \( \bm{d}_t \) is defined from
\begin{equation}\label{eq:eq13}
\begin{array}{ll}
& (\bm{d}_t, \kappa_t) = \displaystyle
\mathop{\mathrm{argmin}}_{\bm{d}, \kappa} \kappa + {\textstyle{1 \over 2}} \|\bm{d}\|_2^2 \\ 
&s.t. \: \Phi_\alpha(\bm{\theta}_t,\mathbf{X})= \nabla [\mathcal{F}_{\mathbf{X}}(\bm{\theta}_t) + \lambda_2\mathcal{H}(W(\bm{\theta}_t))]^{\top} \bm{d}^{(1)} \leq \kappa \\
& \quad \: \: \Phi_\beta(\bm{\theta}_t,\mathbf{W}^c) = \nabla \mathcal{C}(W(\bm{\theta}_t),\mathbf{W}^c)^{\top} \bm{d}^{(2)} \leq \kappa,
\end{array}
\end{equation}
where \( \bm{d}^{(k)} \) denotes the gradient direction of \(k\)-th task, and \(\kappa \in \mathbb{R}\) is a scalar that indicates the convergence status across all tasks. Furthermore, the following proposition holds \cite{fliege2000steepest}

\begin{Corollary}
If \(\bm{\theta}_t\) is Pareto optimal, then it is a stationary point where \(\bm{d}_t = \bm{0}\) and \(\kappa_t = 0\). If \(\bm{\theta}_t\) is not Pareto optimal, then \(\bm{d}_t\) is a valid descent direction, and \(\kappa_t\) is strictly negative, satisfying
\begin{equation}\label{eq:eq14}
\begin{array}{ll}
     &\kappa_t \leq -\textstyle{1 \over 2 } \| \bm{d}_t \|^2_2 \leq 0  \\
     & \Phi_\alpha(\bm{\theta}_t,\mathbf{X}) \leq \kappa_t, \: \Phi_\beta(\bm{\theta}_t,\mathbf{W}^c) \leq \kappa_t.
\end{array}
\end{equation}
\end{Corollary}
Corollary 1 clarifies that when \(\bm{d}_t = \bm{0}\), the data-driven and knowledge-driven objectives cannot be improved simultaneously. Conversely, if \(\bm{\theta}_t\) is not optimal, non-zero \(\bm{d}_t\) guarantees that a direction exists to concurrently improve both objectives. According to KKT condition, it satisfies
\begin{equation}\label{eq:eq15}
\begin{aligned}
    & \bm{d}_t= -\lambda_{\alpha}\Phi_\alpha(\bm{\theta}_t,\mathbf{X}) -\lambda_{\beta}\Phi_\beta(\bm{\theta}_t,\mathbf{W}^c) \\
    & s.t. \quad \lambda_{\alpha} + \lambda_{\beta} = 1.
\end{aligned}
\end{equation}
The dual problem of  Eq.~\eqref{eq:eq15} is
\begin{equation}\label{eq:eq16}
\begin{aligned}
    &     \mathop {\min }\limits_{\lambda_{\alpha}}  -\textstyle{1 \over 2}\| \lambda_{\alpha}\Phi_\alpha(\bm{\theta}_t,\mathbf{X}) +(1-\lambda_{\alpha})\Phi_\beta(\bm{\theta}_t,\mathbf{W}^c) \|^2_2. \\
\end{aligned}
\end{equation}
The quadratic program (QP) presented in Eq.~\eqref{eq:eq16} is equivalent to find the minimum-norm vector in the convex hull of the task gradients. And its solution satisfies (for notational simplicity, we omit the variables in \(\Phi(\cdot)\))\cite{lin2019pareto}:
\begin{equation}\label{eq:eq18}
    {\lambda _\alpha } = \left\{ {\begin{array}{{cc}}
1&{\Phi _\alpha ^{\top}{\Phi _\beta} \ge \Phi _\alpha ^{\top}{\Phi _\alpha }}\\
0&{\Phi _\alpha ^{\top}{\Phi _\beta} \ge \Phi _\beta^{\top}{\Phi _\beta}}\\
{{\textstyle{{{{({\Phi _\beta} - {\Phi _\alpha })}^{\top}}{\Phi _\beta}} \over {\left\| {{\Phi _\alpha } - {\Phi _\beta}} \right\|_2^2}}}}&{\text{otherwise}}.
\end{array}} \right.
\end{equation}

\paragraph{Normalization method.}
Data-driven and knowledge-driven objectives have disparate scales, and the latter requires only sparse parameter modifications and is thus easier to optimize. This imbalance biases QP solution towards neglecting the data-driven task \(\lambda_\alpha \approx 0\).  To ensure both objectives contribute meaningfully, we normalize the gradients in the following ways
\begin{equation}\label{eq:e18}
    \begin{array}{ll}
         & \Phi_{\alpha}= \Phi_{\alpha} \cdot  [( \mathcal{F}_{\mathbf{X}}(\bm{\theta}_t) + \lambda_2 \mathcal{H}(W(\bm{\theta}_t)))\cdot\|\Phi_{\alpha}  \|_2]^{-1}   \\
         & \Phi_{\beta}= \Phi_{\beta} \cdot  [\mathcal{C}(W(\bm{\theta}_t),\mathbf{W}^c) \cdot \|\Phi_{\beta}  \|_2 ]^{-1} .
    \end{array}
\end{equation}
We discuss other normalization methods in Appendix D.

\paragraph{Overall algorithm.} Alg.~\ref{alg} details the RoaDs. It performs a warm-up stage (lines 2-4), using only the data-driven objective for \(t_s\) iterations to find an initial solution. Consistent with the mainstream continuous optimization for causal discovery \cite{yu2019dag, fang2023low}, the main loop uses the Adam optimizer and adjusts the parameters of the acyclicity constraint to accelerate convergence (lines 9-11). We analyze time complexity of Alg.~\ref{alg} in Appendix E.
\begin{algorithm}[tb]
\caption{RoaDs}
    \label{alg}
    \textbf{Input}: Dataset \( \mathbf{X} \), Imperfect priors \(\mathbf{B}^c\). 
    
    \textbf{Output}: Optimal weighted matrix \(\mathbf{\hat W}\). 
    
    \begin{algorithmic}[1] 
        \STATE Align the priors as \(\mathbf{W}^c = \mathcal{M}_{\mathbf{X}}([\mathbf{B}^c]_{\neq 0}) \), set \(\bm{\theta}_0=\mathbf{0}\)
        \WHILE{$t\leq t_s$}
        \STATE \(\bm{\mathbf{\theta_{t+1}}} = \bm{\theta}_{t} + \eta(\Phi_{\alpha}(\bm{\theta}_t,\mathbf{X}))\)
        \ENDWHILE
        \WHILE{$t>t_s$ and $h(W(\bm{\theta}_{t}))\neq0$}
        \STATE Normalize \( \Phi_\alpha(\bm{\theta}_t,\mathbf{X}), \Phi_\beta(\bm{\theta}_t,\mathbf{W}^c) \)
        \STATE Compute \(\lambda_{\alpha}\) and \( \bm{d}_t \)  according to Eq.~ (15) and (17)
        \STATE \(\bm{\mathbf{\theta_{t+1}}} = \bm{\theta}_{t} + \eta\bm{d}_t\)
        \IF{\({h}(W(\bm{\theta}_{t}))>c\cdot{h}(W(\bm{\theta}_{t-1}))\) }
        \STATE Update the parameters in \(\mathcal{H}(W(\bm{\theta}_{t}))\)
        \ENDIF
        \ENDWHILE
        \STATE \textbf{return} the weighted matrix \(\hat{\mathbf{W}}= {W}(\bm{\theta}).\)
    \end{algorithmic}
\end{algorithm}

\section{Experiment}

\subsection{Experimental settings}

\paragraph{Graphs and datasets.}
We generate synthetic graphs using Erd\H{o}s--R\'{e}nyi (ER) and Scale-Free (SF). Each graph consists of \( n_v\) nodes and \( kn_v \) edges, denoted as ER-\(k\) or SF-\(k\). \(n_d\) data is then generated based on SEM defined on these graphs. For linear conditions, the weighted adjacency matrix is sampled randomly from \((-2.0, -0.5] \cup [0.5, 2.0)\). Exogenous noise variables are drawn from Gaussian, Exponential, Gumbel, and Uniform, with settings for both equal variance (EV) and non-equal variance (NV) \cite{ng2024structure}. For nonlinear settings, we generate data using either MLP or Gaussian Processes (GP).

\paragraph{Imperfect constraints usage.}
We sample \(p_a \cdot kn_v\) true edges from the ground-truth graph as positive constraints and \(p_c \cdot p_a \cdot kn_v\) non-existent edges as negative constraints. Then, we randomly select a fraction \(p_b\) of sampled edges and flip their values to simulate imperfect constraints (i.e., a positive constraint is changed to negative, and vice versa).

\paragraph{Baselines and metrics.}
We compare RoaDs against baselines from both continuous and combinatorial methods. The former is founded on GOLEM (linear) and NOTEARS-MLP (nonlinear) \cite{zheng2020learning, ng2020role}. We compare with their extensions under priors, including NTS-B (a type of algorithms incorporating the priors as hard constraints, \cite{sun2023nts,wang2024incorporating}) and ECA \cite{chen2025continuous}. The latter includes PC-stable and LiNGAM \cite{kalisch2007estimating, shimizu2011directlingam}. Performance is evaluated using the F1-score and the Structural Hamming Distance (SHD) against the ground-truth DAGs \cite{zhang2021gcastle}.

\paragraph{Implementation details.}
We set \(s=0.3\) and \(\tau=0.01\) in Eq.~\eqref{eq:eq8}, and the other parameters are default in GOLEM and NOTEARS-MLP. Each experiment was repeated ten times. More details about experimental implementation and code link can be referred in Appendix F.

\subsection{Results and analysis}
\begin{table*}[t]
\centering
\begin{tabular}{l*{8}{c}}
\toprule
\multirow{2}{*}[\multirowoffset]{Method} & \multicolumn{2}{c}{Gauss (ER)} & \multicolumn{2}{c}{Exp (ER)} & \multicolumn{2}{c}{Gauss (SF)}& \multicolumn{2}{c}{Exp (SF)} \\
\cmidrule(lr){2-3} \cmidrule(lr){4-5} \cmidrule(lr){6-7} \cmidrule(lr){8-9}
 & F1($\uparrow$) & SHD($\downarrow$) & F1($\uparrow$) & SHD($\downarrow$) & F1($\uparrow$) & SHD($\downarrow$)  & F1($\uparrow$) & SHD($\downarrow$) \\
\cmidrule(lr){1-9}
PC-stable   & 0.397 & 29.5 & 0.381 & 30.2 & 0.374 & 30.8 & 0.403 & 29.3\\
LiNGAM       & 0.220 & 47.1 & 0.267 & 46.3 & 0.204 & 50.9 & 0.272 & 47.7 \\
NTS-B          & 0.787 & 13.2 & 0.745 & 16.6 & 0.734 & 15.9 & 0.681 & 19.7 \\
ECA          & 0.661 & 24.0 & 0.638 & 25.4 & 0.608 & 26.6 & 0.569 & 29.1 \\
RoaDs (Ours)       & \textbf{0.821} & \textbf{11.4} & \textbf{0.777} & \textbf{14.6} & \textbf{0.750} & \textbf{15.2} & \textbf{0.734} & \textbf{14.1} \\
\cmidrule(lr){1-9}
GOLEM-EV     & 0.807 & 12.1 & 0.728 & 17.4 & 0.701 & 18.2 & 0.672 & 20.0 \\
\bottomrule
\end{tabular}
\caption{Comparison under EV noise (gauss and exp) for linear SEM on the ER-2 and SF-2 ($n_v = 20$, $n_d=2n_v$, $p_a,p_b,p_c = 0.3,0.3,1$) (\(\uparrow\): higher is better, \textbf{bold} indicates the best performance) .}
\label{tab:tab1}
\end{table*}

\begin{table*}[h]
\centering
\begin{tabular}{l*{8}{c}}
\toprule
\multirow{2}{*}[\multirowoffset]{Method} & \multicolumn{2}{c}{Gauss (ER)} & \multicolumn{2}{c}{Exp (ER)} & \multicolumn{2}{c}{Gauss (SF)}& \multicolumn{2}{c}{Exp (SF)} \\
\cmidrule(lr){2-3} \cmidrule(lr){4-5} \cmidrule(lr){6-7} \cmidrule(lr){8-9}
 & F1($\uparrow$) & SHD($\downarrow$) & F1($\uparrow$) & SHD($\downarrow$) & F1($\uparrow$) & SHD($\downarrow$)  & F1($\uparrow$) & SHD($\downarrow$) \\
\cmidrule(lr){1-9}
PC-stable   & \textbf{0.397}  & \textbf{29.5} & 0.381& 30.2 & 0.374 & 30.8 & \textbf{0.403} & \textbf{29.3}\\
LiNGAM       & 0.142 &51.0  & 0.185 &48.4  & 0.124 & 52.9 & 0.161 & 49.1 \\
NTS-B          & 0.300 & 36.6& 0.360 &  32.9 & 0.318 & 33.6 & 0.300 & 35.4 \\
ECA          &  0.362& 38.4& 0.391 & 36.4   & 0.330 & 39.2 & 0.365 & 36.8 \\
RoaDs (Ours)        & 0.384 & 32.7 &\textbf{ 0.434} & \textbf{30.0} & \textbf{0.402} & \textbf{30.2} & 0.370 & 33.2 \\
\cmidrule(lr){1-9}
GOLEM-NV     & 0.301  & 35.4   & 0.336 & 33.9 & 0.281 & 35.0 & 0.371 & 36.3 \\
\bottomrule
\end{tabular}
\caption{Comparison under NV noise for linear SEM on the ER-2 and SF-2 ($n_v = 20$, $n_d=2n_v$, $p_a,p_b,p_c = 0.3,0.3,1$).}
\label{tab:tab2}
\end{table*}

\begin{table*}[h]
\centering
\begin{tabular}{l*{8}{c}}
\toprule
\multirow{2}{*}[\multirowoffset]{Method} & \multicolumn{2}{c}{MLP (ER)} & \multicolumn{2}{c}{GP (ER)} & \multicolumn{2}{c}{MLP (SF)}& \multicolumn{2}{c}{GP (SF)} \\
\cmidrule(lr){2-3} \cmidrule(lr){4-5} \cmidrule(lr){6-7} \cmidrule(lr){8-9}
 & F1($\uparrow$) & SHD($\downarrow$) & F1($\uparrow$) & SHD($\downarrow$) & F1($\uparrow$) & SHD($\downarrow$)  & F1($\uparrow$) & SHD($\downarrow$) \\
\cmidrule(lr){1-9}
PC-stable   & 0.343 & 31.4 & 0.323 & 33.7 & 0.370 & 32.7 & 0.303 & 35.9\\
LiNGAM       & 0.172 & 39.6 & 0.065 & 37.2 & 0.171 & 38.6 & 0.079 & 37.0 \\
NTS-B        & 0.321 & 113.2 & 0.277 & 118.7 & 0.324 & 110.1 & 0.264 & 119.9 \\
ECA          & 0.344 & 107.4 & 0.272 & 119.0 & 0.335 & 106.7 & 0.271 & 118.0 \\
RoaDs (Ours)       & \textbf{0.578} & \textbf{25.9} & \textbf{0.358} & \textbf{32.4} & \textbf{0.520} & \textbf{28.1} & \textbf{0.347} & \textbf{32.9} \\
\cmidrule(lr){1-9}
NOTEARS-MLP     & 0.489 & 31.9 & 0.057 & 35.9 & 0.445 & 30.3 & 0.054 & 35.7 \\
\bottomrule
\end{tabular}
\caption{Comparison under nonlinear SEM on the ER-2 and SF-2 ($n_v = 20$, $n_d=2n_v$, $p_a,p_b,p_c = 0.3,0.3,1$).}
\label{tab:tab3}
\end{table*}

\paragraph{Linear SEM (EV).} As demonstrated in Table \ref{tab:tab1}, imperfect constraints severely mislead the causal discovery, and LiNGAM introduces too many spurious edges to satisfy them. The performance of PC is hampered by the small sample size, which causes less reliable conditional independence tests. The strong performance of continuous optimization methods (ECA, NTS-B, and RoaDs) is attributed to the less non-convex optimization landscape of the linear EV setting \cite{reisach2021beware}. However, NTS-B and ECA rigidly adhere to potentially flawed priors, but RoaDs can harness the benefits of correct priors while resisting misleading ones via prior alignment, resulting in an average F1-score improvement of approximately 4.4\% and 17.0\% decrease in SHD compared to GOLEM-EV.

\paragraph{Linear SEM (NV).}As shown in Table~\ref{tab:tab2}, the linear NV setting introduces a highly non-convex optimization landscape \cite{ng2024structure}, causing a sharp performance decline for most continuous optimization methods. In contrast, PC remains robust as it is less sensitive to noise variances. Notably, our RoaDs maintains performance competitive with PC, demonstrating its superior resilience in navigating this challenging scenario.

\paragraph{Nonlinear SEM.}Under nonlinear conditions (Table \ref{tab:tab3}), PC remains robust due to its non-parametric nature, whereas LiNGAM fails as linearity assumption is violated. NTS-B and ECA, exhibit a significant decline in SHD. They are forced to incorporate an excessive number of edges (over 100) to minimize the least-squares loss while simultaneously adhering to flawed constraints. In this challenging environment, RoaDs achieves remarkable performance, with its F1-score surpassing ECA by an average of 14.5\% and NTS-B by 15.4\%. Furthermore, RoaDs demonstrates its resilience in settings with GP noise, while NOTEARS-MLP achieves a F1-score below 0.1, which indicates a near-complete failure to identify the correct causal edges.

Further comparisons are provided in Appendix~G, covering different noise types (Gumbel and Normal), numbers of variables (\(n_v\)), numbers of edges (\(k\)), and sample sizes (\(n_d\)).

\begin{figure}[htb]
    \centering
    \includegraphics[width=0.9\linewidth]{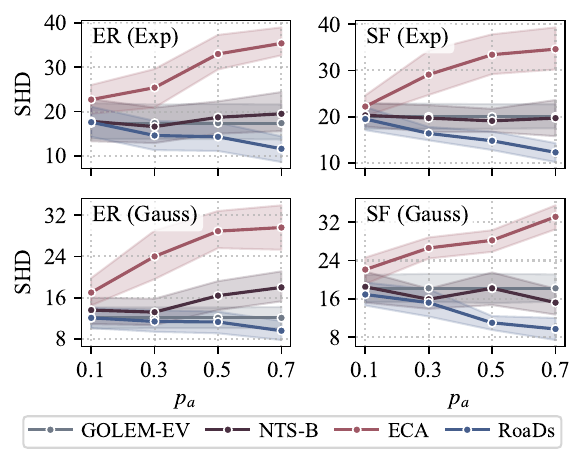}
    \caption{Influence of positive constraints rate \(p_a\) for continuous methods (\(n_v=20,n_d=2n_v,p_b,p_c=0.3,1\)).}
    \label{fig:fig3}
\end{figure}

\begin{figure}[htb]
    \centering
    \includegraphics[width=0.9\linewidth]{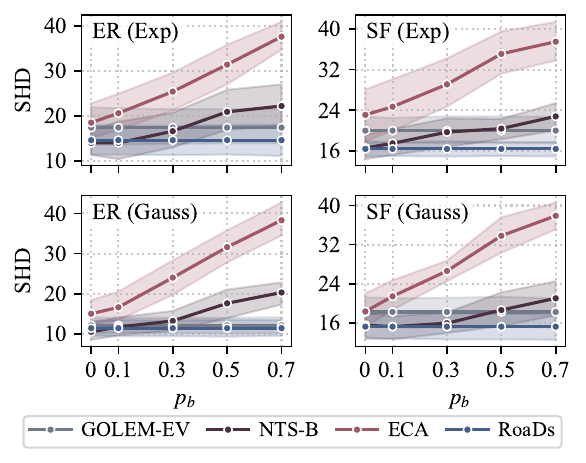}
    \caption{Influence of imperfect constraints rate \(p_b\)  for continuous methods (\(n_v=20,n_d=2n_v,p_a,p_c=0.3,1\)).}
    \label{fig:fig4}
\end{figure}

\paragraph{Influence of constraints.} Figure \ref{fig:fig3} and \ref{fig:fig4} investigate the influence of both the quantity and quality of prior knowledge on continuous optimization methods. When \(p_a\) increases, as more imperfect constraints is introduced, ECA exhibits overfitting to the flawed priors. NTS-B performs comparably to GOLEM-EV. In stark contrast, RoaDs demonstrates the ability to effectively filter this information, as its SHD decreases substantially with a higher \(p_a\). When increasing the error rate \(p_b\) within the constraints, ECA proves highly sensitive, with its SHD increasing dramatically. NTS-B shows a more gradual performance decline. Our proposed RoaDs distinguishes itself by maintaining a stable and low SHD even at high error rates. This superior robustness stems from its prior alignment mechanism, which mitigates the impact of priors that are inconsistent with the observation data.

More detailed comparison is provided in Appendix~H, including results under other settings and sensitivity for \(p_c\).

\begin{figure}[htb]
    \centering
    \includegraphics[width=0.78\linewidth]{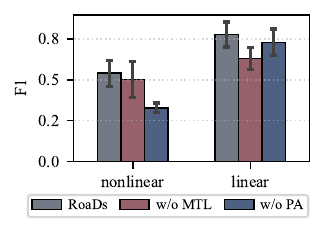}
    \caption{Ablation study on ER-2 (\(n_v=20\), \(n_d=2n_v\),$p_a,p_b,p_c = 0.3,0.3,1$, PA indicates prior alignment).}
    \label{fig:fig5}
\end{figure}

\paragraph{Ablation study. }
Figure~\ref{fig:fig5} presents our ablation study on the contributions of Prior Alignment (PA) and Multi-Task Learning (MTL). In the linear case, MTL is more critical: its removal reduces the F1-score by 14.8\%, whereas removing PA causes only a 4.9\% drop. This suggests that in relatively convex landscapes, effective optimization strategy is more important than the objective's formulation. Conversely, in the highly non-convex nonlinear case, PA becomes dominant. Its removal leads to a 21.1\% F1-score decrease, compared to just 3.8\% for MTL. This indicates that in such complex landscapes, establishing a well-formed optimization objective is more fundamental than the subsequent optimization strategy.

Further evaluation on other components, including different normalization methods, various surrogate models, and running time comparison, is provided in Appendix I.

\begin{table}[ht]
\centering
\begin{tabular}{l*{4}{c}}
\toprule
Method & F1 & SHD & Precison &Recall \\
\cmidrule(lr){1-5}
PC\_stable & 0.333 & 14.0 & 0.384 &0.291  \\
LiNGAM & - & - & - &-  \\
NTS-B   & 0.384& 14.0  & 0.500 &0.235 \\
ECA   &  0.414 & 17.0 &0.500 & \textbf{0.353} \\
RoaDs  & \textbf{0.480} & \textbf{12.0} &0.750 & \textbf{0.353} \\
\cmidrule(lr){1-5}
No Priors & 0.364 & 13.0 & \textbf{0.800} & 0.235  \\
\bottomrule
\end{tabular}
\caption{Comparison under Saches dataset (\(\rm{thres} = 0.1\)).}
\label{tab:tab4}
\end{table}

\paragraph{Case study.} We evaluated our method on the Sachs dataset \cite{sachs2005causal}, a widely-used benchmark for causal discovery from human protein-signaling networks. For our experiments, we used its 853 sample observational data (11 variables) and simulated imperfect domain knowledge with parameters \(p_a,p_b,p_c=0.3,0.3,1\). As summarized in table~\ref{tab:tab4} (for a threshold of 0.1), RoaDs significantly outperforms all competing approaches by achieving the highest F1-score and lowest SHD. Detailed DAG visualizations and results for other thresholds are provided in Appendix J.

\section{Conclusion}
We proposed RoaDs, a novel framework that utilizes the dataset to align priors and employs MTL to resolve the conflict between data-driven and knowledge-driven optimization goals under imperfect structural constraints. Empirical evaluation demonstrates the robustness of RoaDs across both linear (EV and NV) and nonlinear SEMs, as well as its effectiveness under various noise types and constraint rates.

However, this work use MGDA to randomly identify the solution on Pareto front, which may not align with decision-maker's specific preferences. Therefore, future work could focus on developing a Pareto set learning model to generate DAGs adaptable to arbitrary preferences \cite{navonlearning}, or extending RoaDs to incorporate interventional data.

\newpage
\section{Acknowledgments}
This work was supported by the Research Grants Council of the Hong Kong Special Administrative Region, China (GRF Project No. CityU 11215723), by National Natural Science Foundation of China (Project No: 62276223), and by Young Scientists Fund of the National Natural Science Foundation of China (Project No: 52402453).

\section{Appendix}

\subsection{A. Related works for Causal Discovery.}
The field of causal discovery has a long history, initially developing under the name of Bayesian Network Structure Learning \cite{pearl2009causality, koller2009probabilistic}. A primary challenge in early methods was the inability to distinguish causal relationships from statistical dependencies under the standard assumptions of causal sufficiency and faithfulness \cite{glymour2019review}. Consequently, these methods could only identify a MEC, meaning the direction of edges could not be oriented without imposing stricter assumptions.

Traditionally, these approaches are categorized into three families: constraint-based, score-based, and hybrid methods \cite{kitson2023survey}. Constraint-based methods, such as the PC algorithm and its variants, use a series of conditional independence tests to learn the graph's skeleton and orient v-structures, resulting in a Completed Partially Directed Acyclic Graph (CPDAG) \cite{kalisch2007estimating, colombo2014order,le2016fast}. Score-based methods frame causal discovery as a combinatorial optimization problem. They employ search strategies, such as greedy search, evolutionary algorithms, or exact search—within the space of DAGs \cite{larranaga1996structure,de2011efficient, bartlett2017integer,  constantinou2022effective}, CPDAGs \cite{chickering2002optimal,chen2016enumerating,ramsey2017million}, or topological orderings \cite{cooper1992bayesian,teyssier2005ordering,yuan2013learning,scanagatta2015learning} to find the graph that best fits the data. Hybrid methods synergize these two approaches, typically using constraint-based techniques to prune the search space (e.g., identifying parent candidates) before applying a score-based method for final structure optimization \cite{tsamardinos2006max,constantinou2022effective}.

For continuous data, a significant advancement came from methods like LiNGAM \cite{kalisch2007estimating, shimizu2011directlingam}, ANM\cite{hoyer2008nonlinear, buhlmann2014cam}, and PNL \cite{zhang2009identifiability}. By assuming specific functional forms (e.g., linear or non-linear) and non-Gaussian noise, these methods can leverage the resulting model asymmetry to achieve full DAG identification \cite{vowels2022d}.

More recently, the field has seen a surge of interest in continuous optimization techniques. First-order methods reformulate the acyclicity constraint in a differentiable manner, allowing the use of gradient-descent algorithms to find a solution in a continuous space \cite{zheng2018dags, yu2019dag, ng2020role, lachapellegradient,wei2020dags,  yu2021dags, bello2022dagma}. Despite challenges like navigating complex, non-convex landscapes, these methods have achieved highly accurate results. This has spurred numerous improvements, such as extensions for large-scale networks \cite{lopez2022large}, weakened causal sufficiency \cite{cai2023causal,bhattacharya2021differentiable}, weakened causal faithfulness \cite{DBLP:conf/nips/NgZZZ21},  interventional data \cite{brouillard2020differentiable, ke2023neural,DBLP:conf/iclr/DaiNS0LDSZ25}, low-rank settings \cite{DBLP:journals/tnn/FangZZLCH24}, spatial dataset\cite{SunSLP23}, and heterogeneous data \cite{huang2020causal,zhou2025information}. Building on this, second-order methods, leverage the hessian matrix to infer the causal ordering \cite{DBLP:conf/icml/RollandCK0JSL22,reisach2023scale}, with similar work also being explored using diffusion models \cite{sanchezdiffusion}. Separately, other researchers have focused on sampling-based paradigms \cite{charpentierdifferentiable,zhang2023boosting}, employing techniques like Bayesian Optimization \cite{duongcausal}, and reinforcement learning \cite{zhucausal}, to search for the causal graph.

\subsection{B. Proofs.}

\begin{Theorem}
If \(\mathbf{B}^c\) is consistent, there always exists \(\tau > 0\) such that the probability limit of \(\mathbf{W}^c\) from Eq.~\eqref{eq:eq5} satisfies:
\begin{enumerate}
    \item \(\forall X_i \in \bm{V}^{1,0}_j\), then \(\text{plim}_{n_d \to \infty} \mathbf{W}^c_{ij} < \tau\).
    \item \(\forall X_i \in \bm{V}^{1,1}_j\), then \(\text{plim}_{n_d \to \infty} \mathbf{W}^c_{ij} > \tau\).
\end{enumerate}
\end{Theorem}

\begin{proof}
The candidate parents \(\Pi^{\mathcal{G}_c}_j\) of \(X_j\) can be partitioned into two disjoint \(\bm{V}^{1,1}_j\) and \(\bm{V}^{1,0}_j\). The surrogate model aims
\begin{equation}\label{eq:eq1}
    \mathbb{E}[X_j \mid \Pi_{j}^{\mathcal{G}_c}] = \mathbb{E}[X_j \mid \bm{V}^{1,1}_j \cup \bm{V}^{1,0}_j ].
\end{equation}
From the first condition of constraint constraints, we have that \(X_j \Perp \bm{V}^{1,0}_j \mid \bm{V}^{1,1}_j  \), thus, Eq.~\eqref{eq:eq1} can be simplified as
\begin{equation}\label{eq:eq2}
    \mathbb{E}  [X_j \mid \bm{V}^{1,1}_j \cup \bm{V}^{1,0}_j ] = \mathbb{E} [X_j \mid \bm{V}^{1,1}_j],
\end{equation}
which shows that the conditional expectation function is functionally independent of all variables in \(\bm{V}^{1,0}_j\). Since the regressor is consistent, thus, the estimated upper bound converges in probability to the truly weights \(\mathbf{W}^*_{ij} = 0\), i.e.
\begin{equation}\label{eq:eq3}
    \forall \: X_i \in \bm{V}^{1,0}_j, \quad \text{plim}_{n_d \to \infty} \mathbf{W}^p_{ij} = 0.
\end{equation}
This completes the first part of the proof.

Consider the second conclusion. As the \(X_j\) is generated from \(X_j = f_j(\bm{V}^{1,1}_j\cup\bm{V}^{0,1}_j ) + \epsilon_j\), where the surrogate model on \(X_j\) aims to
\begin{equation}\label{eq:eq4}
    \begin{aligned}
    &\mathbb{E}[f_j(\bm{V}^{1,1}_j\cup\bm{V}^{0,1}_j ) + \epsilon_j | \bm{V}^{1,1}_j] \\
    &= \mathbb{E}[(\bm{V}^{1,1}_j\cup\bm{V}^{0,1}_j ) \mid \bm{V}^{1,1}_j] + \mathbb{E}[\epsilon_j \mid \bm{V}^{1,1}_j].
\end{aligned}
\end{equation}
Since \(\epsilon_j \Perp \Pi_j^{\mathcal{G}^*}\) and \(\bm{V}^{1,1}_j \subseteq \Pi_j^{\mathcal{G^*}}\), we have \(\epsilon_j \Perp \bm{V}^{1,1}_j\), which implies \(\mathbb{E}[\epsilon_j \mid \bm{V}^{1,1}_j] = \mathbb{E}[\epsilon_j] = 0\). Thus, Eq.~\eqref{eq:eq4} can be simplified as
\begin{equation}\label{eq:eq5}
    \mathbb{E}[(\bm{V}^{1,1}_j\cup\bm{V}^{0,1}_j ) \mid \mathbf{V}^{1,1}_j].
\end{equation}
Now, we invoke second condition of consistent constraints, which states that \(\bm{V}^{1,1}_j \Perp \bm{V}^{0,1}_j\). This allows us to rewrite the conditional expectation in Eq. \eqref{eq:eq5} as an integral over the marginal distribution of \(\bm{V}^{0,1}_j\), here we denote it as
\begin{equation}\label{eq:eq6}
    g_j(\bm{V}^{1,1}_j) = \int_{\mathcal{X}_{\bm{V}^{0,1}_j}} f_j(\bm{V}^{1,1}_j, \bm{V}^{0,1}_j) p(\bm{V}^{0,1}_j) d\bm{V}^{0,1}_j,
\end{equation}
where \(p(\bm{V}^{0,1}_j)\) is the marginal probability density function of the variables in \(\bm{V}^{0,1}_j\).

For any \(X_l \in \bm{V}^{1,1}_j\), the function \(f_j\) is depend on \(X_l\) in \(\mathcal{G}^*\) (otherwise \(X_l\) would not be a direct cause), and \(g_j\) is a "marginalized" version of \(f_j\). Barring pathological cases where the effect of \(X_l\) is perfectly canceled out by the integration over \(\bm{V}^{0,1}_j\) for all values of \(\bm{V}^{1,1}_j\) (a condition generally excluded by faithfulness assumptions in causal discovery), the function \(g_j\) will also depend on \(X_l\). Therefore, \(g_j\) is not a constant function with respect to any \(X_l \in \mathbf{V}^{1,1}_j\). According to consistency of the regressor
\begin{equation}\label{eq:eq7}
    \forall X_l \in \bm{V}^{1,1}_j, \quad \text{plim}_{n \to \infty} \mathbf{W}^p_{lj} > 0.
\end{equation}
This completes the second part of the proof.

Based on the Eq. \eqref{eq:eq3} and \eqref{eq:eq7}, let
\begin{equation}\label{eq:eq8}
    \tau_{max} = \min_{(l,j) \in \{\mathbf{B}^* = 1\} \cap \{\mathbf{B}^c=1\}} \{\text{plim}_{n_d \to \infty} \mathbf{W}^p_{lj}\}.
\end{equation}
Since this minimum is taken over a finite set of positive numbers, \(\tau_{max} > 0\). Thus, we can choose any threshold \(\tau\) such that \(0 < \tau < \tau_{\max}\). This threshold will asymptotically separate the two sets of edges perfectly.
\end{proof}
\begin{Theorem}\label{thm:thm2}
Consider the continuous optimization problem defined as:
\begin{equation}\label{eq:eq9}
    \min_{\bm{\theta}} \mathcal{F}_\mathbf{X}(\bm{\theta}) + \lambda_2 \mathcal{H}(W(\bm{\theta})) + \lambda_3 \mathcal{C}(W(\bm{\theta}), \mathbf{W}^c).
\end{equation}
Let \(\hat{\bm{\theta}}\) be the optimal solution to the above problem. As the number of samples \( n_d \to \infty \), the graph structure induced by \({W}(\hat{\bm{\theta}})\) converges in probability to the ground-truth DAG \(\mathbf{B}^*\)
\begin{equation}\label{eq:eq10}
    \mathbb{I}({W}(\hat{\bm{\theta}}) > s) \xrightarrow{p} \mathbf{B}^*.
\end{equation}
\end{Theorem}

\begin{proof}
The proof analyzes the first-order necessary conditions for optimality when \(n_d \to \infty\). A parameter vector \(\hat{\bm{\theta}}\) is an optimal solution only if the zero vector is contained in the sub-gradient of the population objective function \(\mathbb{E}[\mathcal{L}(\bm{\theta})]\) evaluated at \(\hat{\bm{\theta}}\), where \(\mathbb{E}[\mathcal{L}(\bm{\theta})]\) is defined as
\begin{equation}\label{eq:eq11}
    \mathbb{E}[\mathcal{L}(\bm{\theta})] = \mathbb{E}[\mathcal{F}_\mathbf{X}(\bm{\theta})] + \lambda_2 \mathcal{H}({W}(\bm{\theta})) + \lambda_3 \mathbb{E}[\mathcal{C}({W}(\bm{\theta}), \mathbf{W}^c)].
\end{equation}
The first-order optimality condition is expressed using the sub-gradient \(\partial_{\bm{\theta}}\)
\begin{equation} \label{eq:eq12}
    \mathbf{0} \in \partial_{\bm{\theta}} \mathbb{E}[\mathcal{L}(\hat{\bm{\theta}})].
\end{equation}
Using the chain rule, we can express the sub-gradient with respect to \(\bm{\theta}\) as
\begin{equation}\label{eq:eq13}
\begin{array}{cc}
     & \partial_{\bm{\theta}} \mathbb{E}[\mathcal{L}] = \nabla_{\bm{\theta}} \mathbb{E}[\mathcal{F}_\mathbf{X}] + \frac{\partial {W}(\bm{\theta})^\top}{\partial \bm{\theta}} ( \lambda_2 \nabla_{{W}} \mathcal{H}({W}(\bm{\theta}))) \\
     & + \lambda_3 \partial_{{W}} \mathbb{E}[\mathcal{C}({W}(\bm{\theta}), \mathbf{W}^c)]
\end{array}
\end{equation}

We now show that condition \eqref{eq:eq10} holds if and only if the DAG of \({W}(\hat{\bm{\theta}})\) corresponds to the true DAG \(\mathbf{B}^*\).

Firstly, consider the necessity. Let \(\hat{\bm{\theta}}\) be a parameterization such that the corresponding graph matches the ground truth, i.e., \(\mathbb{I}({W}_{ij}(\hat{\bm{\theta}}) > s) = \mathbf{B}^*_{ij}\) for all \(i, j\). According to the analysis in NOTEARS \cite{zheng2018dags,zheng2020learning}, \( \mathbf{0} \in \nabla_{{\bm{\theta}}} \mathbb{E}[\mathcal{F}_\mathbf{X}(\hat{\bm{\theta}})]\), and \(W(\hat{\bm{\theta}})\) is naturally acyclic, which implies \( \mathbf{0} \in \nabla_{{{\bm{\theta}}}} \mathcal{H}(W(\hat{\bm{\theta}}))\). Finally, according to Theorem 1, as \(\text{plim}_{n_d \to \infty }\mathbb{I}(\mathbf{W}^c_{ij}> \tau) = \mathbf{B}^*_{ij}\), thus, \(\mathbf{0} \in \partial_{\bm{\theta}} \mathbb{E} (\mathcal{C}(W(\hat{\bm{\theta}}),\mathbf{W}^c)) \) holds. Overall, condition \eqref{eq:eq10} holds.

Next, consider the sufficiency, which proceeds with a proof by contradiction. i.e. \( \mathbf{0} \in \partial_{\bm{\theta}} \mathbb{E}[\mathcal{L}(\hat{\bm{\theta}})]\) while the DAG of \(W(\hat{\bm{\theta}})\) (here we denote as \(\mathcal{G}(\hat{\bm \theta} )\)) contains a false edge or misses a true edge compare with \(\mathbf{B}^*\).

If \(\mathcal{G}(\hat{\bm{\theta}})\) contains a false edge \(X_i\to X_j\) (i.e., \({W}_{ij}(\hat{\bm{\theta}})\) is large but \(\mathbf{B}^*_{ij}=0\)). As this edge violate the data generation process of underlying SEM, thus, \(\nabla_{\bm{\theta}} \mathbb{E}[\mathcal{F}_\mathbf{X}(\hat{\bm{\theta}})] \neq \mathbf{0}\). However, if \(X_j\to X_i\) do not introduce the cycle into \(\mathbf{B}^*\) and is not constrained, both the acyclicity term \(\nabla_{{\bm{\theta}}} \mathcal{H}(W(\hat{\bm{\theta}}))\) and the knowledge term \(\partial_{{\bm{\theta}}} \mathcal{C}(W(\hat{\bm{\theta}}),\mathbf{W}^c)\) is \(\mathbf{0}\). The total sub-gradient \(\partial_{\bm{\theta}} \mathbb{E}[\mathcal{L}(\hat{\bm{\theta}})]\) cannot be zero, which break the assumption.

Similarly, if \(\mathcal{G}(\hat{\bm{\theta}})\) misses a true edge \(X_i\to X_j\) (i.e., \({W}_{ij}(\hat{\bm{\theta}})\) is small but \(\mathbf{B}^*_{ij}=1\)). \(\nabla_{\bm{\theta}} \mathbb{E}[\mathcal{F}_\mathbf{X}(\hat{\bm{\theta}})] \neq \mathbf{0}\) and while \((\nabla_{{\bm{\theta}}} \mathcal{H}(W(\hat{\bm{\theta}})) = 0\) still holds, and if \(\mathbf{B}^c_{ij}=0\), then  \(\partial_{{\bm{\theta}}} \mathcal{C}(W(\hat{\bm{\theta}}),\mathbf{W}^c)=0\) also holds, which implies that \(\partial_{\bm{\theta}} \mathbb{E}[\mathcal{L}(\hat{\bm{\theta}})]\) is not \(\mathbf{0}\).

Overall, the first-order optimality condition for the population objective, \(\mathbf{0} \in \partial_{\bm{\theta}} \mathbb{E}[\mathcal{L}(\bm{\theta})]\), is satisfied exclusively at a parameterization \(\hat{\bm{\theta}}\) where the corresponding weight matrix \({W}(\hat{\bm{\theta}})\) represents the ground-truth DAG \(\mathbf{B}^*\). Since the minimizer of the empirical objective converges to the minimizer of the population objective, the learned graph structure \(\mathbb{I}({W}(\hat{\bm{\theta}}) > s)\) converges in probability to \(\mathbf{B}^*\).
\end{proof}

\subsection{C. A Special Case.}

When all \( f_j \) are linear, the model becomes a linear SEM, \( \mathbf{X} = \mathbf{XW} + \bm{\epsilon} \), where \( \mathbf{W} \in \mathbb{R}^{n_v \times n_v} \) is the weighted adjacency matrix. If the noise \(\bm{\epsilon}\) is non-Gaussian, the causal structure \(\mathbf{W}\) is identifiable \cite{kalisch2007estimating, shimizu2011directlingam}. The original causal discovery task can be simplified as  \cite{zheng2018dags,ng2020role}:
\begin{equation}\label{eq:eq14}
\begin{array}{l}
\mathop {\min }\limits_\mathbf{W} \: \frac{{{1}}}{n_d} \left\| {\mathbf{X} - \mathbf{XW}} \right\|_F^2  + \lambda_1 \|\mathbf{W}\|_1 \\
s.t. \: h(\mathbf{W}) = \mathrm{tr}(e^{\mathbf{W} \circ \mathbf{W}}) - n_v = 0,
\end{array}
\end{equation}
where \(h(\mathbf{W})\) is a differentiable acyclic constraint. 

In this specific context, the surrogate model can be achieved by the consistent parametric regressor (e.g. linear regression, lasso regression), where the regression coefficient of \(\mathcal{M}^{(j)}_\mathbf{X}(\mathbf{B}_{ij})\) can be directly used to approximate the truly weights in \(\mathbf{W}\). Consequently, the non-linear sigmoid mapping \(\sigma(\cdot)\) is no longer required, and \(s=\tau\) holds. Thus, knowledge-driven optimization objective admits a significant simplification
\begin{equation}\label{eq:eq15}
     \min_{\mathbf{W}} \left\| ( \mathbf{W} - \mathbf{W}^c  ) \circ \mathbf{B}^c \right\|_1,
\end{equation}
where \(\mathbf{W}^c\) is the OLS estimate:
\begin{equation}\label{eq:eq8}
    (\mathbf{W}^c)_{ij} = \left\{ {\begin{array}{ll}
        {{(\mathbf{x}_i^{\top}{\mathbf{x}_i})}^{ - 1}}\mathbf{x}_i^{\top}{\mathbf{x}_j} & {\text{if } (\mathbf{B}^c)_{ij} = 1}\\
        0 & {\text{if } (\mathbf{B}^c)_{ij} = 0}
    \end{array}} \right.
\end{equation}
As a hard incorporation, we posit that formulation \eqref{eq:eq15} is more robust than other hard but fixed-thresholding methods in \cite{wang2024incorporating}, especially when the prior knowledge \(\mathbf{B}^c\) is imperfect.
The following theorem formalizes this claim by comparing the final estimators produced by each approach.

\begin{Theorem}
    Let \(\mathbf{W}^*\) be the ground-truth weight matrix and \(\mathbf{B}^c\) be an imperfect constraint matrix. For a linear non-Gaussian SEM, consider two estimators for \(\mathbf{W}\):
    \begin{equation}\label{eq:eq17}
    \begin{array}{ll}
         & \hat{\mathbf{W}}^c = \mathop{\arg \min}_{\mathbf{W}} \mathcal{F}'_{\mathbf{X}}(\mathbf{W}) + \lambda_3 \left\| (\mathbf{W} - \mathbf{W}^c) \circ \mathbf{B}^c \right\|_1 \\
         & \hat{\mathbf{W}}^p = \mathop{\arg \min}_{\mathbf{W}} \mathcal{F}'_{\mathbf{X}}(\mathbf{W}) + \lambda_3 \left\| \rm{relu}(\mathbf{W}^p - |\mathbf{W}_{ij}| )  \circ \mathbf{B}^c \right\|_1,
    \end{array}
    \end{equation}
    where \(\mathcal{F}'_{\mathbf{X}}(\mathbf{W}) = \mathcal{F}_\mathbf{X}(\mathbf{W}) + \lambda_2\mathcal{H}(\mathbf{W})\). If \(\mathbf{B}^c\) is consistent, then the estimation error of \(\hat{\mathbf{W}}^c\) is less than or equal to that of \(\hat{\mathbf{W}}^p\) when \(n_d \to \infty\):
    \begin{equation}\label{eq:eq18}
    \forall X_i,X_j \in \bm{V},
        |\hat{\mathbf{W}}^c_{ij} - \mathbf{W}^*_{ij}| \le |\hat{\mathbf{W}}^p_{ij}  - \mathbf{W}^*_{ij}|.
    \end{equation}
\end{Theorem}

\begin{proof}
    Let \(\mathbf{W}^\#\) denote the optimal solution for \(\mathcal{F}'_\mathbf{X}(\mathbf{W})\). From the first-order optimality conditions, the sub-gradient of \(\mathcal{F}'_{\mathbf{X}}\) at \(\mathbf{W}^\#\) must contain the zero vector
    \begin{equation}\label{eq:eq19}
        \mathbf{0} \in \partial \mathcal{F}'_{\mathbf{X}}(\mathbf{W}) |_{\mathbf{W}=\mathbf{W}^\#} = \nabla \mathcal{F}_\mathbf{X}(\mathbf{W}^\#) + \lambda_2 \partial \mathcal{H}(\mathbf{W}^\#).
    \end{equation}
    Thus, we focus on how the third term perturbs the solution from \(\mathbf{W}^\#\). The analysis proceeds with an element-wise consideration of an arbitrary weight \(\mathbf{W}_{ij}\).
    
    When prior constraint is inactive (\(\mathbf{B}^c_{ij} = 0\)), the third regularization term in both objective functions is nullified by the Hadamard product with zero, and 
    \begin{equation}\label{eq:eq20}
        \hat{\mathbf{W}}^c_{ij} = \hat{\mathbf{W}}^p_{ij} = \mathbf{W}^\#_{ij}.
    \end{equation}
    Consequently, their estimation errors are equal, and the inequality in the theorem holds as an equality.
    
    When prior constraint is active (\(\mathbf{B}^c_{ij} = 1\)), the sub-gradient of objective \eqref{eq:eq17}-1 can be written as: 
    \begin{equation}\label{eq:eq21}
        \partial \mathcal{F}'_\mathbf{X}(\mathbf{W}_{ij}) |_{\mathbf{W}=\mathbf{W}^\#} + \lambda_3 \cdot \partial |\mathbf{W}_{ij} - \mathbf{W}^c_{ij}| |_{\mathbf{W}=\mathbf{W}^\#}
    \end{equation}
    From Eq. \eqref{eq:eq19}, the sub-gradient of the first term contains 0. When the dataset is not infinite, according to Theorem 2, \(\mathbf{W}^\#_{ij} \neq \mathbf{W}^c_{ij}\) holds. The sub-gradient of the second term is \( \lambda_3 \cdot \text{sgn}(\mathbf{W}^\#_{ij} - \mathbf{W}^c_{ij})\). Thus, the solution must move from \(\mathbf{W}^\#_{ij}\) in the direction opposite to the gradient, i.e., towards the \(\mathbf{W}^c_{ij}\).
    
    However, the sub-gradient of \eqref{eq:eq17}-2 can be formulated as
    \begin{equation}\label{eq:eq22}
        \partial \mathcal{F}'_\mathbf{X}(\mathbf{W}_{ij}) |_{\mathbf{W}=\mathbf{W}^\#} + \lambda_3 \cdot \partial \rm{max}(0, \mathbf{W}^p_{ij}- |\mathbf{W}_{ij}| ) |_{\mathbf{W}=\mathbf{W}^\#}.
    \end{equation}
    This penalty acts as a lower bound, penalizing any value of \(\mathbf{W}_{ij}\) that is not achieve the fixed prior \(\mathbf{W}^p_{ij} = s\), where \(s\) is the threshold for edge presence.
    
    However, when the \(\mathbf{B}^c\) is imperfect, Eq. \eqref{eq:eq21} pushes the \(\mathbf{W}^\#\) towards the \(\mathbf{W}^c\), which is asymptotically consistent with the ground-truth DAG \(\mathbf{W}^*\) according to Theorem 1. Thus,  \( |\hat{\mathbf{W}}^p_{ij}  - \mathbf{W}^*_{ij}| \to 0 \) holds when \(n_d \to \infty\).
    In contrast, for Eq. \eqref{eq:eq22}, there \(\exists X_i,X_j, \: s.t. \: \mathbf{B}^*_{ij} = 1,\mathbf{W}^p_{ij} = 0\) or \(\exists X_i,X_j, \: s.t. \: \mathbf{B}^*_{ij} = 0,\mathbf{W}^p_{ij} = s\), which forced \(\mathbf{W}^\#\) far away from \(\mathbf{W}^*\), results in \(|\hat{\mathbf{W}}^p_{ij}  - \mathbf{W}^*_{ij}| > s\) when \(n_d \to \infty\).
    
    By synthesizing the analyses for the case where the constraint is active \(\mathbf{B}^c = 1\) and inactive \(\mathbf{B}^c = 0\), the following inequality is shown to hold universally
    \begin{equation}
    \forall X_i,X_j \in \bm{V},
        |\hat{\mathbf{W}}^c_{ij} - \mathbf{W}^*_{ij}| \le |\hat{\mathbf{W}}^p_{ij}  - \mathbf{W}^*_{ij}|.
    \end{equation}
\end{proof}

\subsection{D. Normalization methods.}
In this paper, we consider the following four types of normalization methods for the gradients of different tasks.
\begin{enumerate}
    \item  \textit{L2}: Normalizes each gradient to a unit vector, retaining only its direction.
    \begin{equation}\label{eq:eq24}
        \Phi_{\alpha}= \frac{\Phi_{\alpha}}{\|\Phi_{\alpha}  \|_2  }, \Phi_{\beta}= \frac{\Phi_{\beta} }{ \|\Phi_{\beta}  \|_2  }.
    \end{equation}
    \item \textit{Loss}: Scales each gradient by the loss value of the objective.
    \begin{equation}\label{eq:eq25}
        \Phi_{\alpha}= \frac{\Phi_{\alpha} }{ \mathcal{F}_{\mathbf{X}}(\bm{\theta}_t) + \lambda_2 \mathcal{H}(W(\bm{\theta}_t))  }, \Phi_{\beta}= \frac{\Phi_{\beta} }{ \mathcal{C}(W(\bm{\theta}_t),\mathbf{W}^c)  }.
    \end{equation}
    \item \textit{Loss+}: Combines the above methods, normalizing by both the gradient's L2-norm and the objective's loss value.
    \begin{equation}\label{eq:e26}
        \begin{array}{ll}
             & \Phi_{\alpha}= \Phi_{\alpha} \cdot  [( \mathcal{F}_{\mathbf{X}}(\bm{\theta}_t) + \lambda_2 \mathcal{H}(W(\bm{\theta}_t)))\cdot\|\Phi_{\alpha}  \|_2]^{-1}   \\
             & \Phi_{\beta}= \Phi_{\beta} \cdot  [\mathcal{C}(W(\bm{\theta}_t),\mathbf{W}^c) \cdot \|\Phi_{\beta}  \|_2 ]^{-1} .
        \end{array}
    \end{equation}
    \item \textit{None}: Uses the primarily value of \(\Phi_\alpha,\Phi_\beta\).
\end{enumerate}
It is worth noting that \textit{L2} normalization simplifies the multi-objective optimization with a equal weighted sum scalarization, as it invariably results in \(\lambda_\alpha=0.5\). In experiments, we evaluated four normalization methods and selected \textit{Loss+}, which demonstrated the best performance.

\subsection{E. Complexity Analysis.}
The time complexity of RoaDs is analyzed in two distinct stages: prior alignment and MTL optimization. The first stage involves a one-time pre-computation to fit the surrogate regressor. The complexity of this step is dependent on the chosen model, for instance, \(O(kn_v^2)\) for linear regression. The second stage's computational bottleneck remains the gradient calculation for the acyclicity constraint \(h(W(\bm{\theta}))\), which has a complexity of \(O(n_v^3)\) \cite{zheng2018dags,ng2020role,zheng2020learning}. Other operations, such as computing the value of \(\mathcal{C}(W(\bm{\theta}), \mathbf{W}^c)\) and executing the MGDA solver, have a lower complexity of \(O(n_v^2)\).

Thus, the complexity of RoaDs is still dominated by the acyclicity constraint. Assuming the one-time cost of the selected surrogate model does not exceed this bound, the overall iterative complexity is \(O(n_v^3)\).

\subsection{F. Experimental Settings.}

\paragraph{Graphs, Datasets and Constraints.} Synthetic datasets were generated using the \texttt{gcastle} library\footnote{https://github.com/huawei-noah/trustworthyAI/tree/master/gcastle}, based on ER and SF graph structures. For the linear SEM, the scale of equal variance (EV) noise (Gauss, Exp, Gumbel, Normal) was set to 1, and non-equal variance (NV) noise was obtained by standardizing the data under EV noise \cite{reisach2021beware, ng2024structure}. For non-linear SEMs, the functional relationships (MLP and GP) were modeled using the default parameters within \texttt{gcastle}. The real-world Sachs dataset was obtained from the \texttt{bnlearn repository}\footnote{https://www.bnlearn.com/research/sachs05/} \cite{sachs2005causal}. The values of these parameters are detailed in Table \ref{tab:tab1}.

\begin{table}[ht]
\centering
\small 
\setlength{\tabcolsep}{4pt} 
\begin{tabularx}{\columnwidth}{l>{\raggedright\arraybackslash}Xl} 
\toprule
Notation & Meanings & Value \\
\midrule
\(n_v\) & Number of nodes & \(\{20,40\}\) \\
\(k\) & Ratio of edges & \({1,2,4}\) \\
\(n_d\) & Number of dataset & \(2n_v, 4n_v\) \\
\(p_a\) & Positive constraints rate & \(\{0.1,0.3,0.5,0.7\}\) \\
\(p_c\) & Negative constraints ratio & \(0,1,2\) \\
\(p_b\) & Imperfect priors rate & \(\{0,0.1,0.3,0.5,0.7\}\) \\
\bottomrule
\end{tabularx}
\caption{Experiment settings on graphs, datasets and constraints.}
\label{tab:tab1}
\end{table}

\paragraph{Baselines.} All baseline methods were implemented using the \texttt{gcastle} library. We selected two foundational continuous optimization algorithms: GOLEM for linear SEMs and NOTEARS-MLP for non-linear SEMs. GOLEM optimizes the following objective:
\begin{equation}
    \begin{array}{l}
    \mathop {\min }\limits_\mathbf{W} {\cal L}(\mathbf{W};\mathbf{X}) - \log \left| {\det (\mathbf{I} - \mathbf{W})} \right| + {\lambda _1}{\left\| \mathbf{W} \right\|_1} + {\lambda _2}h(\mathbf{W})\\
    {{\cal L}_{EV}}(\mathbf{W};\mathbf{X}) = \frac{{{n_v}}}{2}\log \left\| {\mathbf{X} - \mathbf{X}\mathbf{W}} \right\|_F^2\\
    {{\cal L}_{NV}}(\mathbf{W};\mathbf{X}) = \frac{1}{2}\sum\limits_{i = 1}^{{n_v}} {\log \left\| {{\mathbf{X}_{:,i}} - \mathbf{X}{\mathbf{W}_{:,i}}} \right\|_2^2}.
    \end{array}
\end{equation}
In accordance with \cite{ng2024structure}, we set the sparsity penalty \(\lambda_1\)
to 0.2 for the EV noise and 0.1 for the NV noise. A larger sparsity penalty is crucial for small sample sizes to mitigate overfitting to least square loss, thereby avoiding the inclusion of superfluous edges. The models were trained for a maximum of 10,000 iterations using the Adam optimizer. Other parameters were set to default values in \texttt{gcastle}. NOTEARS-MLP aims to optimize Eq. (2) in main paper, and we still set the sparsity penalty as \(\lambda_1 = 0.1\), while the remaining parameters were left at default settings.

We employed stable version of PC \cite{colombo2014order}. Fisher's Z-test was utilized for the conditional independence tests, with \(\alpha=0.05\). For the LiNGAM family, we used DirectLiNGAM, a method based on iterative regression and residual comparisons \cite{shimizu2011directlingam}.

The NTWS-B minimizes the following objective function\cite{wang2024incorporating}:\
\begin{equation}
     \mathop {\min }\limits_{\mathbf{W}} \lambda_3(\left\| \rm{relu}(s\cdot\mathbf{B}^c_{=1} - |\mathbf{W}|_{ij} )  \circ \mathbf{B}^c_{=1} \right\|_1 + \left\| |\mathbf{W}_{ij}|\circ \mathbf{B}^c_{=0} \right\|_1).
\end{equation}
We set \(s=0.3\), \(\lambda_3=\lambda_1\) to align with the sparsity penalization. These settings were retained for the non-linear case.

The CEA  minimizes the following objective function \cite{chen2025continuous}:
\begin{equation}
\begin{array}{ll}
     & \mathop {\min }\limits_\mathbf{W} -\xi^2 \left\| {{\mathbf{B}^c} \circ \log ({\mathbf{W}'} \circ {\mathbf{W}^p} + (1 - {\mathbf{W'}}) \circ (1 - {\mathbf{W}^p}))} \right\|_\Sigma   \\
     & \mathbf{W}' = |2\sigma(\mathbf{W}) - 1|
\end{array}
\end{equation}

Following the recommendations from the source paper, we set \(\xi=1\). The prior knowledge matrix \(\mathbf{W}^p \in \{e_p,e_a \}^{n_v \times n_v},\) was configured with \(e_p=0.9\) to represent positive edge constraints and \(e_a=0.1\) for negative edge constraints. For nonlinear case, an additional sparsity term governed by \(\lambda_3\) was incorporated into the objective function.

Finally, thresholding was set as 0.3 and applied to the absolute edge weights to convert them into a binary graph.

\paragraph{Metrics.}

The metrics for evaluating the graphical accuracy involve F1 and SHD \cite{kitson2023survey}. F1 is defined as 
\[Prec\text{=}\frac{\mathrm{TP}}{\mathrm{TP+FP}},Rec=\frac{\mathrm{TP}}{\mathrm{TP+FN}},F1=2\frac{Prec\centerdot Rec}{Prec+Rec}\]
where TP denote the number of directed edges correctly identified in the learned DAG that also exist in the ground-truth benchmark DAG. FP correspond to spurious edges present in the learned DAG but absent in the benchmark DAG, while FN represent edges in the benchmark DAG that are missing in the learned DAG. SHD is the sum of the number of superfluous edges, missing edges and reversed edges.

\paragraph{Implementation details.} For the linear case, we employed linear regression and lasso regression. For the non-linear case, we utilized polynomial regression and random forest regression. For former, terms were considered up to the third degree to maintain the time complexity of RoaDs below \(O(n_v^3)\). The latter was configured with 100 trees, and its feature importance was evaluated using permutation importance, with 10 repetitions for each permutation. Across all scenarios, the number of warm-up iterations was set to 10. Furthermore, to mitigate overfitting in the non-linear setting, we did not apply the constrained weight matrix mask.

\subsection{G. Main Results.}

Tables \ref{tab:tab2} and \ref{tab:tab3} extend the Tables 1 and 2 in main paper, summarize the performance of RoaDs under other two noise types (Gumbel and Uniform, both EV and NV), respectively.
Under the EV condition, RoaDs consistently achieves the highest F1 score, outperforming NTS-B and ECA by an average of 4.1\% and 16.2\%, respectively. Conversely, in the NV setting, PC-stable delivers the best performance among the evaluated algorithms, with our proposed RoaDs demonstrating comparable results.
These findings are largely consistent with the conclusions drawn from Tables 1 and 2 in the main paper.

\paragraph{Effect of num of edges.}
Tables \ref{tab:tab4} through \ref{tab:tab9} detail the comparative performance of the algorithms on graphs with varying densities (ER-1, ER-4, SF-1, and SF-4) under diverse SEM and noise configurations.

In linear setting, while NTS-B consistently achieves a lower SHD than RoaDs in denser graphs (ER-4, SF-4), but RoaDs maintains a superior F1-score compared to both NTS-B and ECA. Conversely, in sparser graphs (ER-1, SF-1), RoaDs demonstrates exceptional robustness, securing the highest F1 score and the lowest SHD in 7 out of 8 conditions. We hypothesize that in highly dense graphs, the accuracy of the prior alignment process may decrease, leading to erroneous prior estimates and less precise final DAGs.

In the non-linear setting, RoaDs continues to deliver compelling performance across both F1 and SHD criteria, regardless of graph density. As other continuous optimization methods tend to overfit the least-squares objective, often producing overly dense graphs with SHD scores exceeding 100. Therefore, RoaDs establishes itself as a robust and reliable choice for non-linear causal discovery with priors.

\paragraph{Effect of size of dataset.}
Tables \ref{tab:tab10}, \ref{tab:tab11}, and \ref{tab:tab12} detail the algorithmic performance with an increased sample size of \(n_d=4n_v\). Nevertheless, RoaDs consistently maintains its superior performance across both linear and non-linear settings. The provision of more data further enhances its accuracy, with the F1 score improving by over 1.6\% in the linear case and 4.5\% in the non-linear case compared to the results from the smaller dataset \(n_d=2n_v\).

\paragraph{Effect of num of nodes.}
Tables \ref{tab:tab13} and \ref{tab:tab14} present a comparative analysis of the algorithms' performance under \(n_v=40\). LiNGAM fails to produce an acyclic structure, likely due to disturbances from the imperfect constraints. In this scenario, the advantages of RoaDs become more pronounced. It outperforms NTS-B with an average F1 score improvement of 3.6\% and surpasses ECA by 14.4\%. Furthermore, RoaDs achieves a significant reduction in SHD, averaging 8.6\% lower than that of GOLEM-EV.

These findings indicate that RoaDs remains a robust and superior choice for causal discovery, even when applied to larger-scale problems.

\subsection{H. Influence of Priors.}
This subsection evaluates the performance of continuous-based methods under different rates of positive edge constraints \(p_a\), negative edge constraints \(p_c\) and flawed constraints \(p_b\). All experiments were conducted with a fixed setup of \(n_v=20,n_d=2n_v,k=2\). To analyze each parameter's effect, one rate was varied while the others were held at baseline values (specifically, \(p_a=0.3\), \(p_b=0.3\) and \(p_c = 1\)). Note that in the non-linear setting, both ECA and NTS-B fail to estimate the true DAG accurately. Consequently, our comparative analysis is focused on PC-stable and NOTEARS-MLP.

\paragraph{Positive edge constraints rate. } Figures \ref{fig:fig1} through \ref{fig:fig4} illustrate how the SHD of the evaluated methods changes as \(p_a\) increase. In both linear SEM settings (EV and NV), the performance of RoaDs improves progressively with the quantity of available prior knowledge. The method's prior alignment mechanism effectively identifies and utilizes these constraints, leading to a progressive decrease in the SHD of the learned DAG. The results on nonlinear conditon demonstrate that RoaDs consistently outperforms PC-stable. 

\paragraph{Imperfect constraints rate. } Figures \ref{fig:fig5} through \ref{fig:fig8} demonstrate the algorithmic performance under an increasing of \(p_b\). ECA proves to be highly sensitive to incorrect priors and its SHD sharp increase when more flawed constraints are introduced. In contrast, the other three algorithms exhibit greater stability and RoaDs consistently maintains the lowest SHD. This robustness is particularly evident in the non-linear condition. While PC-stable also shows sensitivity to flawed priors, the proposed RoaDs method sustains a remarkably stable performance.

\paragraph{Negative constraints ratio. } Figures \ref{fig:fig9} through \ref{fig:fig13} show the SHD distribution under varying ratios of negative constraints \(p_c=0,1,2\). In the linear setting, RoaDs is robust to the composition of prior knowledge. Regardless of whether negative constraints are absent or abundant, RoaDs consistently outperforms ECA and NTS-B, while also maintaining a slight advantage over GOLEM. However, in the non-linear condition, the performance of RoaDs becomes less stable. Too many negative constraints disturb the prior alignment, which leads to an incorrect estimation of the weight matrix \(\mathbf{W}^c\), resulting in a notable performance decrease.

\subsection{I. Other Parameter Experiments.}

\paragraph{Running Time.} Figure \ref{fig:fig14} presents the convergence times for each method in the nonlinear setting. PC-Stable is the fastest and consistently terminates in 5s. Its efficiency is derived from its non-iterative nature. The runtime of NOTEARS-MLP is affected by the noise type: it converges rapidly under GP noise, but it produces a trivial, near-empty graph (see Table 3 in main paper). Where it successfully learns a DAG(e.g., MLP noise), its runtime is notably longer than that of our proposed method. NTS-B and ECA both employ a equal weighted-sum scalarization to combine the objectives and are the most computationally intensive. RoaDs is significantly more efficient than them, and its advantage stems from the use of the MGDA. Instead of relying on scalarization, MGDA computes a common descent direction that guarantees simultaneous improvement for both objectives. This leads to a more direct and faster convergence path, consistently reducing the overall computational time.

\paragraph{Normalization methods.} Tables \ref{tab:tab15} and \ref{tab:tab16} compare the influence of different normalization methods on the RoaDs algorithm in the linear case. The results demonstrate that normalizing the terms \(\Phi_\alpha, \Phi_\beta\) using both the loss function value and the L2-norm of its gradient achieves the best performance.

\paragraph{Surrogate models.} Tables \ref{tab:tab17} and \ref{tab:tab18} compare the performance of different surrogate models used in the prior alignment process. In the linear case, standard linear regression outperforms Lasso regression, which is because the priors already implicitly contain sparsity information. In the non-linear setting, random forest regression proves superior to polynomial regression, as the non-parametric nature of random forests allows for greater flexibility and accuracy.

\subsection{J. Case Study.}

The true causal graph of Saches is sourced from \texttt{BN Repository} \footnote{https://www.bnlearn.com/bnrepository/} \cite{sachs2005causal}, depicted in the top-left panel of Figure \ref{fig:fig15}. From left to right, these nodes represent \textit{Raf, Mek, Plcg, PIP2, PIP3, Erk, Akt, PKA, PKC, P38,} and \textit{Jnk}.
For this experiment, we introduced imperfect constraints with parameters \(p_a=0.3,p_b=0.3,p_c=1\), as illustrated in the second panel of Figure \ref{fig:fig15}. GOLEM consistently converged to the same DAG regardless of the threshold used (Figure \ref{fig:fig15}, top row, third panel). LiNGAM failed to learn a valid DAG, while the output of PC-stable is also shown (Figure \ref{fig:fig15}, top row, fourth panel).

The performance of continuous optimization methods was evaluated in remain rows of Figure \ref{fig:fig15} with \(s=0.05,0.1,0.2,0.3\). A quantitative comparison was summarized in Table \ref{tab:tab19}. Unsurprisingly, RoaDs achieves a stable SHD of 13, matching the performance achieved without imperfect constraints. Furthermore, it obtains the highest F1 score for \(s=0.05,0.1,0.2\), indicating its superior ability to learn reliable causal relationships from observational data even when provided with flawed prior knowledge.

\begin{figure}[htbp]
    \centering
    \includegraphics[width=0.9\linewidth]{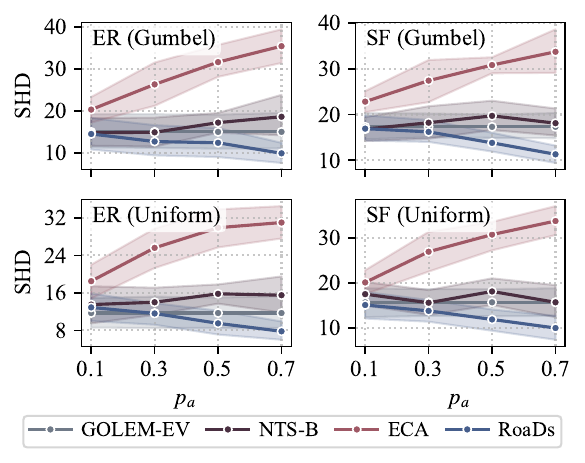}
    \caption{SHD of continuous-based methods under different \(p_a\) (linear SEM with gumbel and uniform noise (EV)).}
    \label{fig:fig1}
\end{figure}

\begin{figure}[htbp]
    \centering
    \includegraphics[width=0.9\linewidth]{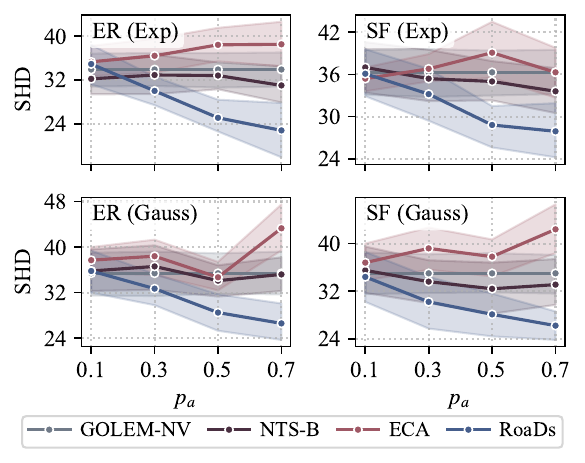}
    \caption{SHD of continuous-based methods under different \(p_a\) (linear SEM with exp and gauss noise (NV)).}
    \label{fig:fig2}
\end{figure}

\begin{figure}[htbp]
    \centering
    \includegraphics[width=0.9\linewidth]{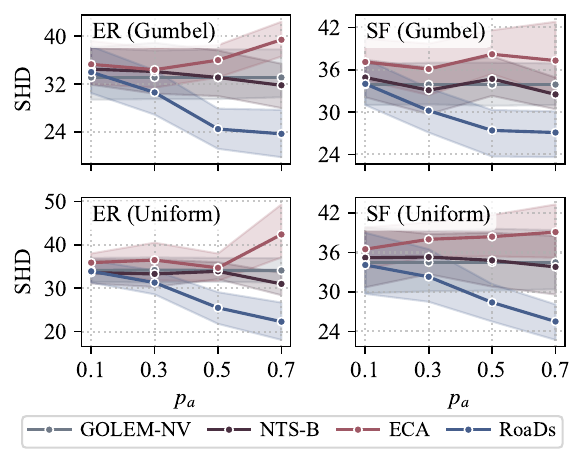}
    \caption{SHD of continuous-based methods under different \(p_a\) (linear SEM with gumbel and uniform noise (NV)).}
    \label{fig:fig3}
\end{figure}

\begin{figure}[htbp]
    \centering
    \includegraphics[width=0.9\linewidth]{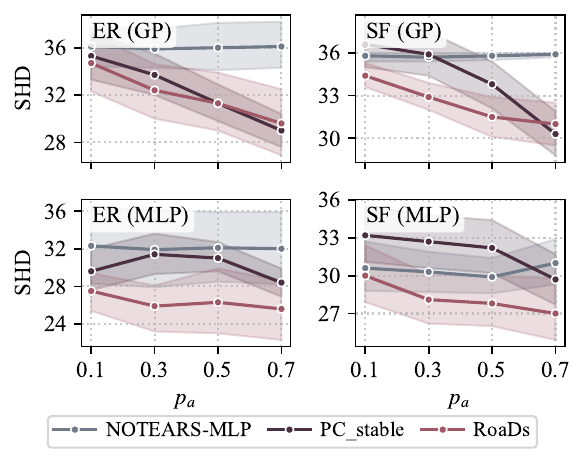}
    \caption{SHD of continuous-based methods and PC-stable under different \(p_a\) (nonlinear SEM).}
    \label{fig:fig4}
\end{figure}

\begin{table*}[htbp]
\centering
\footnotesize
\begin{tabular}{l*{8}{c}}
\toprule
\multirow{2}{*}[\multirowoffset]{Method} & \multicolumn{2}{c}{Gumbel (ER)} & \multicolumn{2}{c}{Uniform (ER)} & \multicolumn{2}{c}{Gumbel (SF)}& \multicolumn{2}{c}{Uniform (SF)} \\
\cmidrule(lr){2-3} \cmidrule(lr){4-5} \cmidrule(lr){6-7} \cmidrule(lr){8-9}
 & F1($\uparrow$) & SHD($\downarrow$) & F1($\uparrow$) & SHD($\downarrow$) & F1($\uparrow$) & SHD($\downarrow$)  & F1($\uparrow$) & SHD($\downarrow$) \\
\cmidrule(lr){1-9}
PC\_stable   & 0.416 & 28.2 & 0.389 & 28.9 & 0.383 & 30.1 & 0.398 & 30.0\\
LiNGAM       & 0.229 & 46.8 & 0.227 & 45.5 & 0.235 & 50.5 & 0.232 & 51.1 \\
NTS-B          & 0.760 & 14.9 & 0.776 & 14.0 & 0.689 & 18.2 & 0.744 & 15.6 \\
ECA          & 0.633 & 26.3 & 0.633 & 25.6 & 0.605 & 27.4 & 0.613 & 26.9 \\
RoaDs (Ours)        & \textbf{0.805} & \textbf{12.7} & \textbf{0.818} & \textbf{11.6} & \textbf{0.732} & \textbf{16.2} & \textbf{0.777} & \textbf{13.8} \\
\cmidrule(lr){1-9}
GOLEM-EV     & 0.752 & 15.0 & 0.813 & 11.7 & 0.706 & 17.3 & 0.731 & 15.7 \\
\bottomrule
\end{tabular}
\caption{Extension Table 1 in main paper (\(\uparrow\): higher is better, \textbf{bold}: best performance among algorithms that incorporates priors).}
\label{tab:tab2}
\end{table*}

\begin{table*}[htbp]
\centering
\small
\begin{tabular}{l*{8}{c}}
\toprule
\multirow{2}{*}[\multirowoffset]{Method} & \multicolumn{2}{c}{Gumbel (ER)} & \multicolumn{2}{c}{Uniform (ER)} & \multicolumn{2}{c}{Gumbel (SF)}& \multicolumn{2}{c}{Uniform (SF)} \\
\cmidrule(lr){2-3} \cmidrule(lr){4-5} \cmidrule(lr){6-7} \cmidrule(lr){8-9}
 & F1($\uparrow$) & SHD($\downarrow$) & F1($\uparrow$) & SHD($\downarrow$) & F1($\uparrow$) & SHD($\downarrow$)  & F1($\uparrow$) & SHD($\downarrow$) \\
\cmidrule(lr){1-9}
PC\_stable   & \textbf{0.416} & \textbf{28.2} & 0.389 & \textbf{28.9} & 0.383 & 30.1 & \textbf{0.398} & \textbf{30.0}\\
LiNGAM       & 0.153 & 50.0 & 0.167 & 47.7 & 0.151 & 54.9 & 0.122 & 52.8 \\
NTS-B          & 0.323 & 34.1 & 0.355 & 33.3 & 0.351 & 33.1 & 0.317 & 35.3 \\
ECA          & 0.382 & 34.4  & 0.360 & 36.5 & 0.358 & 36.1 & 0.338 & 38.0 \\
RoaDs (Ours)        & 0.410 & 30.6 & \textbf{0.395} & 31.3 & \textbf{0.412} & \textbf{30.0} & 0.393 & 32.3 \\
\cmidrule(lr){1-9}
GOLEM-NV     & 0.351 & 33.1  & 0.342 & 34.2 & 0.311 & 33.9 & 0.311 & 34.5 \\
\bottomrule
\end{tabular}
\caption{Extension Table 2 in main paper.}
\label{tab:tab3}
\end{table*}

\begin{table*}[htbp]
\centering
\small
\begin{tabular}{l*{8}{c}}
\toprule
\multirow{2}{*}[\multirowoffset]{Method} & \multicolumn{2}{c}{Gauss} & \multicolumn{2}{c}{Exp} & \multicolumn{2}{c}{Gumbel} & \multicolumn{2}{c}{Uniform} \\
\cmidrule(lr){2-3} \cmidrule(lr){4-5} \cmidrule(lr){6-7} \cmidrule(lr){8-9}
 & F1($\uparrow$) & SHD($\downarrow$) & F1($\uparrow$) & SHD($\downarrow$) & F1($\uparrow$) & SHD($\downarrow$)  & F1($\uparrow$) & SHD($\downarrow$) \\

\cmidrule(lr){1-9}
PC\_stable   & 0.360 & 63.3 & 0.341 & 65.3 & 0.352 & 64.6 & 0.364 & 63.7 \\
LiNGAM       & 0.163 & 79.9 & 0.212 & 78.9 & 0.240 & 78.6 & 0.174 & 79.3 \\
NTS-B          & 0.545 & \textbf{51.2} & 0.534 & \textbf{53.2} & 0.542 & \textbf{52.9} & \textbf{0.574} & \textbf{48.9} \\
ECA          & 0.493 & 61.0 & 0.503 & 62.6 & 0.492 & 62.3 & 0.488 & 62.0 \\
RoaDs (Ours)      & \textbf{0.559} & 52.1 & \textbf{0.555} & 53.4 & \textbf{0.566} & 53.1 & 0.561 & 52.9 \\
\cmidrule(lr){1-9}
GOLEM-EV    & 0.546 & 50.4  & 0.531 & 52.8  & 0.555 & 50.4  & 0.553 & 49.7 \\
\bottomrule
\end{tabular}
\caption{Comparison for the linear SEM (EV) on the ER-4 ($n_v = 20$, $n_d=2n_v$, $p_a,p_b,p_c = 0.3,0.3,1$).}
\label{tab:tab4}
\end{table*}

\begin{table*}[htbp]
\centering
\small
\begin{tabular}{l*{8}{c}}
\toprule
\multirow{2}{*}[\multirowoffset]{Method} & \multicolumn{2}{c}{Gauss} & \multicolumn{2}{c}{Exp} & \multicolumn{2}{c}{Gumbel} & \multicolumn{2}{c}{Uniform} \\
\cmidrule(lr){2-3} \cmidrule(lr){4-5} \cmidrule(lr){6-7} \cmidrule(lr){8-9}
 & F1($\uparrow$) & SHD($\downarrow$) & F1($\uparrow$) & SHD($\downarrow$) & F1($\uparrow$) & SHD($\downarrow$)  & F1($\uparrow$) & SHD($\downarrow$) \\
\cmidrule(lr){1-9}
PC\_stable   & 0.339 & 56.3 & 0.342 & 56.4 & 0.354 & 56.0 & 0.345 & 57.3 \\
LiNGAM       & 0.167 & 68.3 & 0.178 & 68.2 & 0.174 & 67.6 & 0.199 & 67.6 \\
NTS-B          & 0.581 & \textbf{41.7} & 0.506 & \textbf{47.6} & 0.556 & \textbf{44.2} & 0.569 & \textbf{42.1} \\
ECA          & 0.495 & 51.8 & 0.492 & 53.8 & 0.491 & 54.3 & 0.498 & 53.9 \\
RoaDs (Ours)      & \textbf{0.590} & 44.1 & \textbf{0.528} & 49.4 & \textbf{0.572} & 45.9 & \textbf{0.582} & 45.1 \\
\cmidrule(lr){1-9}
GOLEM-EV      & 0.540 & 44.1 & 0.527 & 45.7 & 0.546 & 44.4 & 0.551 & 43.0 \\
\bottomrule
\end{tabular}
\caption{Comparison for the linear SEM (EV) on the SF-4 ($n_v = 20$, $n_d=2n_v$, $p_a,p_b,p_c = 0.3,0.3,1$).}
\label{tab:tab5}
\end{table*}

\begin{table*}[htbp]
\centering
\small
\begin{tabular}{l*{8}{c}}
\toprule
\multirow{2}{*}[\multirowoffset]{Method} & \multicolumn{2}{c}{Gauss} & \multicolumn{2}{c}{Exp} & \multicolumn{2}{c}{Gumbel}& \multicolumn{2}{c}{Uniform} \\
\cmidrule(lr){2-3} \cmidrule(lr){4-5} \cmidrule(lr){6-7} \cmidrule(lr){8-9}
 & F1($\uparrow$) & SHD($\downarrow$) & F1($\uparrow$) & SHD($\downarrow$) & F1($\uparrow$) & SHD($\downarrow$)  & F1($\uparrow$) & SHD($\downarrow$) \\
\cmidrule(lr){1-9}
PC\_stable   & 0.429 & 13.8 & 0.459 & 13.4 & 0.453 & 13.4 & 0.470 & 12.8 \\
LiNGAM       & 0.302 & 23.1 & 0.416 & 22.9 & 0.386 & 23.3 & 0.414 & 20.5 \\
NTS-B          & \textbf{0.864} & \textbf{4.60} & 0.697 & 10.4 & \textbf{0.821} & \textbf{5.90} & 0.852 & 5.10 \\
ECA          & 0.739 & 9.80 & 0.653 & 13.7 & 0.746 & 9.20 & 0.766 & 9.00 \\
RoaDs (Ours)        & 0.861 & \textbf{4.60} & \textbf{0.735} & \textbf{9.20} & 0.812 & 6.10 & \textbf{0.866} & \textbf{4.70} \\
\cmidrule(lr){1-9}
GOLEM-EV      & 0.858 & 4.80 & 0.740 & 8.90 & 0.812 & 6.10 & 0.864 & 4.80 \\
\bottomrule
\end{tabular}
\caption{Comparison for the linear SEM (EV) on the ER-1 ($n_v = 20$, $n_d=2n_v$, $p_a,p_b,p_c = 0.3,0.3,1$).}
\label{tab:tab6}
\end{table*}

\begin{table*}[htbp]
\centering
\small
\begin{tabular}{l*{8}{c}}
\toprule
\multirow{2}{*}[\multirowoffset]{Method} & \multicolumn{2}{c}{Gauss} & \multicolumn{2}{c}{Exp} & \multicolumn{2}{c}{Gumbel}& \multicolumn{2}{c}{Uniform} \\
\cmidrule(lr){2-3} \cmidrule(lr){4-5} \cmidrule(lr){6-7} \cmidrule(lr){8-9}
 & F1($\uparrow$) & SHD($\downarrow$) & F1($\uparrow$) & SHD($\downarrow$) & F1($\uparrow$) & SHD($\downarrow$)  & F1($\uparrow$) & SHD($\downarrow$) \\
\cmidrule(lr){1-9}
PC\_stable   & 0.434 & 13.8 & 0.442 & 13.8 & 0.476 & 12.8 & 0.483 & 12.6 \\
LiNGAM       & 0.260 & 27.7 & 0.369 & 24.6 & 0.363 & 22.9 & 0.271 & 27.7 \\
NTS-B          & 0.838 & 5.60 & 0.695 & 11.2 & 0.746 & 8.70 & 0.797 & 6.90 \\
ECA          & 0.704 & 11.6 & 0.622 & 15.7 & 0.651 & 14.1 & 0.724 & 10.4 \\
RoaDs (Ours)        & \textbf{0.840} & \textbf{5.20} & \textbf{0.702} & \textbf{10.8} & \textbf{0.788} & \textbf{7.20} & \textbf{0.835} & \textbf{5.70} \\
\cmidrule(lr){1-9}
GOLEM-EV     & 0.839 & 5.20 & 0.697 & 11.0 & 0.786 & 7.30 & 0.843 & 5.50 \\
\bottomrule
\end{tabular}
\caption{Comparison for the linear SEM (EV) on the SF-1 ($n_v = 20$, $n_d=2n_v$, $p_a,p_b,p_c = 0.3,0.3,1$).}
\label{tab:tab7}
\end{table*}

\begin{table*}[htbp]
\centering
\small
\begin{tabular}{l*{8}{c}}
\toprule
\multirow{2}{*}[\multirowoffset]{Method} & \multicolumn{2}{c}{MLP (ER)} & \multicolumn{2}{c}{GP (ER)} & \multicolumn{2}{c}{MLP (SF)}& \multicolumn{2}{c}{GP (SF)} \\
\cmidrule(lr){2-3} \cmidrule(lr){4-5} \cmidrule(lr){6-7} \cmidrule(lr){8-9}
 & F1($\uparrow$) & SHD($\downarrow$) & F1($\uparrow$) & SHD($\downarrow$) & F1($\uparrow$) & SHD($\downarrow$)  & F1($\uparrow$) & SHD($\downarrow$) \\
\midrule
PC\_stable   & 0.329 & 64.6 & 0.315 & 65.9 & 0.316 & 55.9 & 0.327 & \textbf{57.4} \\
LiNGAM       & 0.207 & 74.7 & 0.044 & 74.1 & 0.156 & 54.0 & 0.042 & 64.5 \\
NTS-B          & 0.490 & 106.8 & 0.416 & 114.3 & 0.464 & 105.3 & 0.370 & 120.0 \\
ECA          & 0.507 & 102.2 & \textbf{0.433} & 112.8 & 0.468 & 104.5 & \textbf{0.390} & 116.8 \\
RoaDs (Ours)        & \textbf{0.512} & \textbf{56.1} & 0.340 & \textbf{65.6} & \textbf{0.530} & \textbf{48.0} & 0.307 & 59.5 \\
\cmidrule(lr){1-9}
NOTEARS-MLP    & 0.424 & 63.3 & 0.026 & 73.9 & 0.425 & 54.8 & 0.045 & 63.9 \\
\bottomrule
\end{tabular}
\caption{Comparison for nonlinear SEM on the ER-4 and SF-4 ($n_v = 20$, $n_d=2n_v$, $p_a,p_b,p_c = 0.3,0.3,1$).}
\label{tab:tab8}
\end{table*}

\begin{table*}[htbp]
\centering
\small
\begin{tabular}{l*{8}{c}}
\toprule
\multirow{2}{*}[\multirowoffset]{Method} & \multicolumn{2}{c}{MLP (ER)} & \multicolumn{2}{c}{GP (ER)} & \multicolumn{2}{c}{MLP (SF)}& \multicolumn{2}{c}{GP (SF)} \\
\cmidrule(lr){2-3} \cmidrule(lr){4-5} \cmidrule(lr){6-7} \cmidrule(lr){8-9}
 & F1($\uparrow$) & SHD($\downarrow$) & F1($\uparrow$) & SHD($\downarrow$) & F1($\uparrow$) & SHD($\downarrow$)  & F1($\uparrow$) & SHD($\downarrow$) \\
\midrule
PC\_stable   & 0.387 & 16.7 & 0.256 & 21.1 & 0.423 & 16.2 & 0.274 & 22.1 \\
LiNGAM       & 0.224 & 21.2 & 0.132 & 20.5 & 0.172 & 23.0 & 0.157 & 21.3 \\
NTS-B          & 0.212 & 108.0 & 0.152 & 118.1 & 0.227 & 106.0 & 0.177 & 116.4 \\
ECA          & 0.208 & 109.0 & 0.164 & 117.4 & 0.217 & 106.3 & 0.189 & 113.9 \\
RoaDs (Ours)        & \textbf{0.572} & \textbf{14.8} & \textbf{0.385} & \textbf{16.3} & \textbf{0.651} & \textbf{11.6} & \textbf{0.368} & \textbf{16.7} \\
\cmidrule(lr){1-9}
NOTEARS-MLP     & 0.535 & 15.0 & 0.137 & 18.6 & 0.540 & 15.1 & 0.145 & 18.8 \\
\bottomrule
\end{tabular}
\caption{Comparison for nonlinear SEM on the ER-1 and SF-1 ($n_v = 20$, $n_d=2n_v$, $p_a,p_b,p_c = 0.3,0.3,1$).}
\label{tab:tab9}
\end{table*}

\begin{table*}[htbp]
\centering
\small
\begin{tabular}{l*{8}{c}}
\toprule
\multirow{2}{*}[\multirowoffset]{Method} & \multicolumn{2}{c}{Gauss} & \multicolumn{2}{c}{Exp} & \multicolumn{2}{c}{Gumbel}& \multicolumn{2}{c}{Uniform} \\
\cmidrule(lr){2-3} \cmidrule(lr){4-5} \cmidrule(lr){6-7} \cmidrule(lr){8-9}
 & F1($\uparrow$) & SHD($\downarrow$) & F1($\uparrow$) & SHD($\downarrow$) & F1($\uparrow$) & SHD($\downarrow$)  & F1($\uparrow$) & SHD($\downarrow$) \\
\cmidrule(lr){1-9}
PC\_stable   & 0.443 & 27.2 & 0.451 & 27.1 & 0.437 & 27.8 & 0.440 & 27.7 \\
LiNGAM       & 0.187 & 48.7 & 0.366 & 40.9 & 0.341 & 43.2 & 0.313 & 45.0 \\
NTS-B          & 0.793 & 12.6 & 0.746 & 15.4 & 0.763 & 14.7 & 0.766 & 14.1 \\
ECA          & 0.683 & 21.3 & 0.632 & 26.1 & 0.666 & 23.4 & 0.692 & 21.3 \\
RoaDs (Ours)        & \textbf{0.829} & \textbf{10.6} & \textbf{0.771} & \textbf{14.6} & \textbf{0.826} & \textbf{10.9} & \textbf{0.815} & \textbf{11.3} \\
\cmidrule(lr){1-9}
GOLEM-EV     & 0.818 & 11.3 & 0.750 & 15.3 & 0.796 & 12.6 & 0.791 & 12.3 \\
\bottomrule
\end{tabular}
\caption{Comparison for the linear SEM (EV) on the ER-2 ($n_v = 20$, $n_d=4n_v$, $p_a,p_b,p_c = 0.3,0.3,1$).}
\label{tab:tab10}
\end{table*}

\begin{table*}[htbp]
\centering
\small
\begin{tabular}{l*{8}{c}}
\toprule
\multirow{2}{*}[\multirowoffset]{Method} & \multicolumn{2}{c}{Gauss} & \multicolumn{2}{c}{Exp} & \multicolumn{2}{c}{Gumbel}& \multicolumn{2}{c}{Uniform} \\
\cmidrule(lr){2-3} \cmidrule(lr){4-5} \cmidrule(lr){6-7} \cmidrule(lr){8-9}
 & F1($\uparrow$) & SHD($\downarrow$) & F1($\uparrow$) & SHD($\downarrow$) & F1($\uparrow$) & SHD($\downarrow$)  & F1($\uparrow$) & SHD($\downarrow$) \\
\cmidrule(lr){1-9}
PC\_stable   & 0.369 & 30.9 & 0.389 & 30.6 & 0.417 & 29.6 & 0.426 & 29.3 \\
LiNGAM       & 0.178 & 51.6 & 0.245 & 52.4 & 0.248 & 51.2 & 0.285 & 49.6 \\
NTS-B          & 0.765 & 13.9 & 0.722 & 17.0 & 0.725 & 15.9 & 0.765 & 14.1 \\
ECA          & 0.637 & 24.3 & 0.627 & 25.8 & 0.620 & 26.0 & 0.637 & 25.5 \\
RoaDs (Ours)        & \textbf{0.797} & \textbf{12.7} & \textbf{0.748} & \textbf{15.7} & \textbf{0.788} & \textbf{13.2} & \textbf{0.812} & \textbf{11.6} \\
\cmidrule(lr){1-9}
GOLEM-EV     & 0.760 & 14.3 & 0.738 & 16.4 & 0.764 & 14.3 & 0.782 & 13.4 \\
\bottomrule
\end{tabular}
\caption{Comparison for linear SEM (EV) on the SF-2 ($n_v = 20$, $n_d=4n_v$, $p_a,p_b,p_c = 0.3,0.3,1$).}
\label{tab:tab11}
\end{table*}

\begin{table*}[htbp]
\centering
\small
\begin{tabular}{l*{8}{c}}
\toprule
\multirow{2}{*}[\multirowoffset]{Method} & \multicolumn{2}{c}{MLP (ER)} & \multicolumn{2}{c}{GP (ER)} & \multicolumn{2}{c}{MLP (SF)}& \multicolumn{2}{c}{GP (SF)} \\
\cmidrule(lr){2-3} \cmidrule(lr){4-5} \cmidrule(lr){6-7} \cmidrule(lr){8-9}
 & F1($\uparrow$) & SHD($\downarrow$) & F1($\uparrow$) & SHD($\downarrow$) & F1($\uparrow$) & SHD($\downarrow$)  & F1($\uparrow$) & SHD($\downarrow$) \\
\cmidrule(lr){1-9}
PC\_stable   & 0.391 & 31.3 & 0.338 & 34.0 & 0.392 & 31.2 & 0.309 & 35.0 \\
LiNGAM       & 0.154 & 40.2 & 0.071 & 35.4 & 0.183 & 38.2 & 0.061 & 35.9 \\
NTS-B          & 0.362 & 102.3 & 0.280 & 102.1 & 0.392 & 94.2 & 0.328 & 92.6 \\
ECA          & 0.372 & 100.7 & 0.307 & 102.0 & 0.391 & 95.3 & 0.312 & 101.7 \\
RoaDs (Ours)        & \textbf{0.633} & \textbf{24.1} & \textbf{0.375} & \textbf{32.4} & \textbf{0.607} & \textbf{24.1} & \textbf{0.367} & \textbf{32.2} \\
\cmidrule(lr){1-9}
NOTEARS-MLP     & 0.568 & 26.3 & 0.055 & 35.8 & 0.549 & 25.5 & 0.078 & 35.8 \\
\bottomrule
\end{tabular}
\caption{Comparison for nonlinear SEM on the ER-2 and SF-2 ($n_v = 20$, $n_d=4n_v$, $p_a,p_b,p_c = 0.3,0.3,1$).}
\label{tab:tab12}
\end{table*}

\begin{table*}[htbp]
\centering
\small
\begin{tabular}{l*{8}{c}}
\toprule
\multirow{2}{*}[\multirowoffset]{Method} & \multicolumn{2}{c}{Gauss} & \multicolumn{2}{c}{Exp} & \multicolumn{2}{c}{Gumbel}& \multicolumn{2}{c}{Uniform} \\
\cmidrule(lr){2-3} \cmidrule(lr){4-5} \cmidrule(lr){6-7} \cmidrule(lr){8-9}
 & F1($\uparrow$) & SHD($\downarrow$) & F1($\uparrow$) & SHD($\downarrow$) & F1($\uparrow$) & SHD($\downarrow$)  & F1($\uparrow$) & SHD($\downarrow$) \\
\cmidrule(lr){1-9}
PC\_stable   & 0.359 & 70.1 & 0.369 & 69.9 & 0.360 & 70.1 & 0.367 & 69.7 \\
LiNGAM       & / & 80.3 & / & 80.3 & / & 80.3 & / & 80.3 \\
NTS-B          & 0.783 & 31.3 & 0.723 & 40.9 & 0.745 & 36.8 & 0.747 & 36.0 \\
ECA          & 0.602 & 63.1 & 0.589 & 66.6 & 0.593 & 64.1 & 0.612 & 62.1 \\
RoaDs (Ours)        & \textbf{0.790} & \textbf{29.8} & \textbf{0.756} & \textbf{35.1} & \textbf{0.781} & \textbf{31.8} & \textbf{0.799} & \textbf{29.1} \\
\cmidrule(lr){1-9}
GOLEM-EV    & 0.747 & 35.3 & 0.747 & 36.2 & 0.750 & 35.9 & 0.764 & 33.6 \\
\bottomrule
\end{tabular}
\caption{Comparison for the linear SEM (EV) on the ER-2 ($n_v = 40$, $n_d=2n_v$, $p_a,p_b,p_c = 0.3,0.3,1$).}
\label{tab:tab13}
\end{table*}

\begin{table*}[htbp]
\centering
\small
\begin{tabular}{l*{8}{c}}
\toprule
\multirow{2}{*}[\multirowoffset]{Method} & \multicolumn{2}{c}{Gauss} & \multicolumn{2}{c}{Exp} & \multicolumn{2}{c}{Gumbel}& \multicolumn{2}{c}{Uniform} \\
\cmidrule(lr){2-3} \cmidrule(lr){4-5} \cmidrule(lr){6-7} \cmidrule(lr){8-9}
 & F1($\uparrow$) & SHD($\downarrow$) & F1($\uparrow$) & SHD($\downarrow$) & F1($\uparrow$) & SHD($\downarrow$)  & F1($\uparrow$) & SHD($\downarrow$) \\
\cmidrule(lr){1-9}
PC\_stable   & 0.352 & 66.1 & 0.336 & 67.4 & 0.330 & 68.1 & 0.336 & 67.5 \\
LiNGAM       & / & 76.0 & / & 76.0 & / & 76.0 & / & 76.0 \\
NTS-B          & 0.665 & 43.6 & 0.626 & 48.8 & 0.625 & 48.4 & 0.658 & 43.6 \\
ECA          & 0.557 & 63.7 & 0.542 & 66.9 & 0.558 & 64.2 & 0.569 & 60.9 \\
RoaDs (Ours)       & \textbf{0.691} & \textbf{38.8} & \textbf{0.669} & \textbf{42.6} & \textbf{0.669} & \textbf{43.3} & \textbf{0.707} & \textbf{38.1} \\
\cmidrule(lr){1-9}
GOLEM-EV    & 0.661 & 41.9 & 0.630 & 47.1 & 0.646 & 44.4 & 0.677 & 40.4 \\
\bottomrule
\end{tabular}
\caption{Comparison for linear SEM (EV) on the SF-2 ($n_v = 40$, $n_d=2n_v$, $p_a,p_b,p_c = 0.3,0.3,1$).}
\label{tab:tab14}
\end{table*}

\begin{figure}[htbp]
    \centering
    \includegraphics[width=0.9\linewidth]{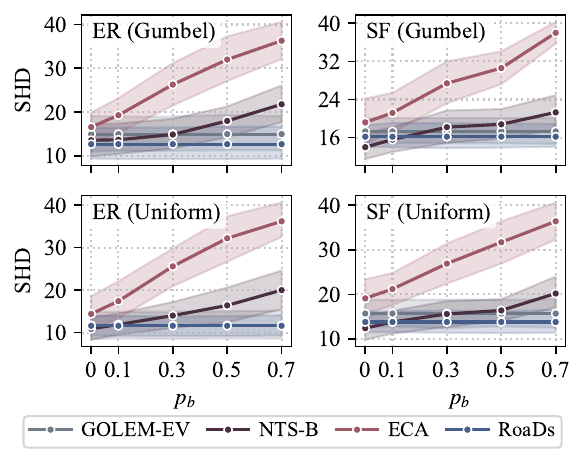}
    \caption{SHD of continuous-based methods under different \(p_b\) (linear SEM with gumbel and uniform noise (EV)).}
    \label{fig:fig5}
\end{figure}

\begin{figure}[htbp]
    \centering
    \includegraphics[width=0.9\linewidth]{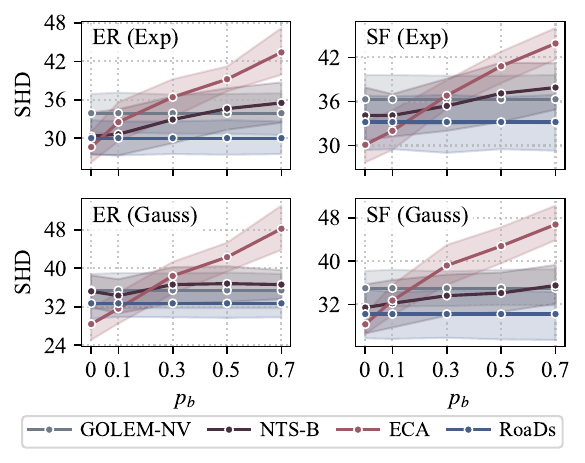}
    \caption{SHD of continuous-based methods under different \(p_b\) (linear SEM with exp and gauss noise (NV)).}
    \label{fig:fig6}
\end{figure}

\begin{figure}[htbp]
    \centering
    \includegraphics[width=0.9\linewidth]{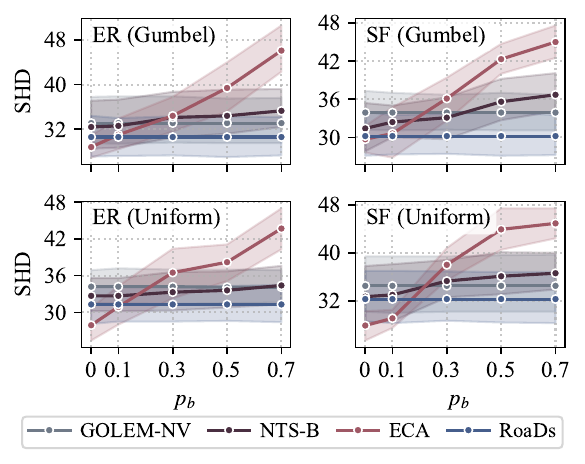}
    \caption{SHD of continuous-based methods under different \(p_b\) (linear SEM with gumbel and uniform noise (NV)).}
    \label{fig:fig7}
\end{figure}

\begin{figure}[htbp]
    \centering
    \includegraphics[width=0.9\linewidth]{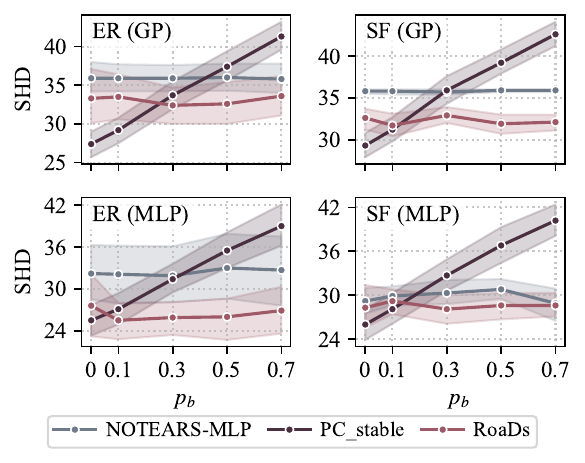}
    \caption{SHD of continuous-based methods and PC-stable under different \(p_b\) (nonlinear SEM).}
    \label{fig:fig8}
\end{figure}

\begin{figure}[htbp]
    \centering
    \includegraphics[width=0.9\linewidth]{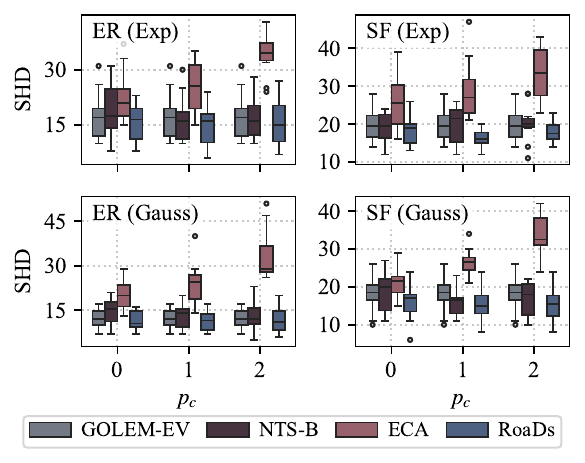}
    \caption{SHD of continuous-based methods under different \(p_c\) (linear SEM with exp and gauss noise (EV)).}
    \label{fig:fig9}
\end{figure}

\begin{figure}[htbp]
    \centering
    \includegraphics[width=0.9\linewidth]{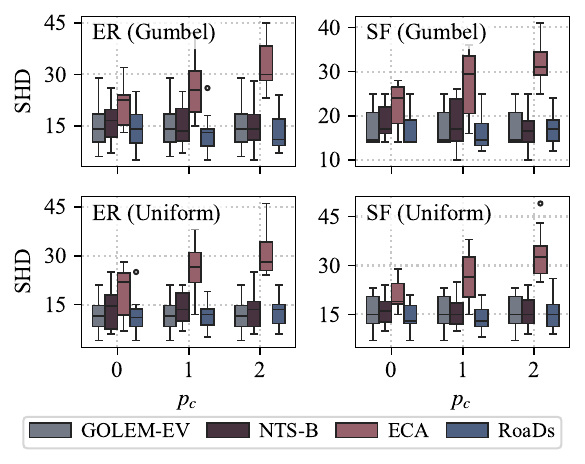}
    \caption{SHD of continuous-based methods under different \(p_c\) (linear SEM with gumbel and uniform noise (EV)).}
    \label{fig:fig10}
\end{figure}

\begin{figure}[htbp]
    \centering
    \includegraphics[width=0.9\linewidth]{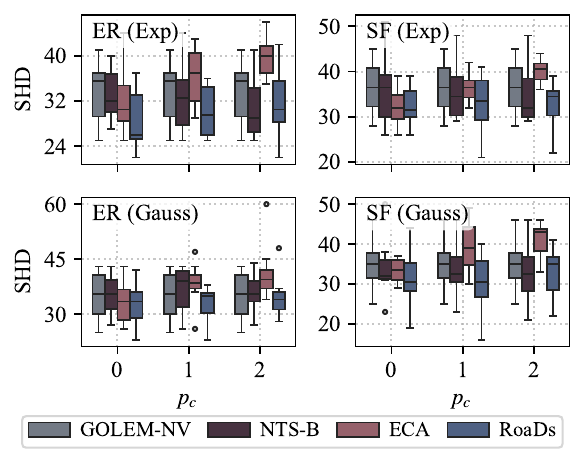}
    \caption{SHD of continuous-based methods under different \(p_c\) (linear SEM with exp and gauss noise (NV)).}
    \label{fig:fig11}
\end{figure}

\begin{figure}[htbp]
    \centering
    \includegraphics[width=0.9\linewidth]{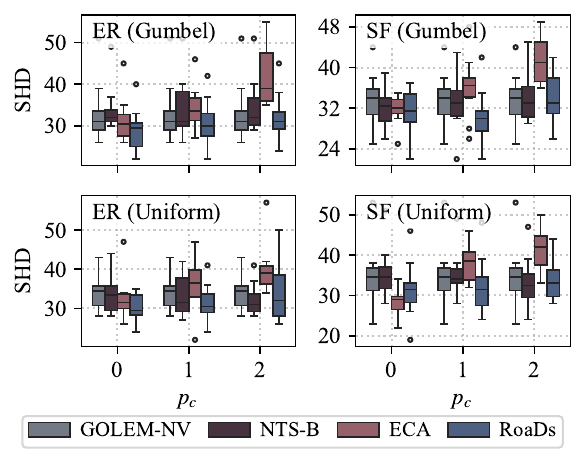}
    \caption{SHD of continuous-based methods under different \(p_c\) (linear SEM with gumbel and uniform noise (NV)).}
    \label{fig:fig12}
\end{figure}

\begin{figure}[htbp]
    \centering
    \includegraphics[width=0.9\linewidth]{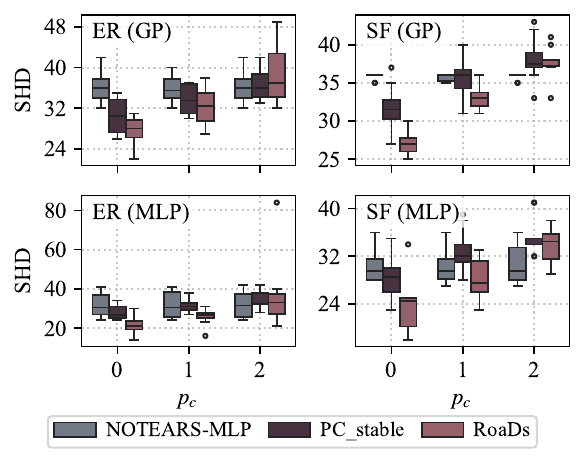}
    \caption{SHD of continuous-based methods and PC-stable under different \(p_c\) (nonlinear SEM).}
    \label{fig:fig13}
\end{figure}

\begin{table*}[htbp]
\centering
\begin{tabular}{l*{8}{c}}
\toprule
\multirow{2}{*}[\multirowoffset]{Method} & \multicolumn{2}{c}{Gauss} & \multicolumn{2}{c}{Exp} & \multicolumn{2}{c}{Gumbel} & \multicolumn{2}{c}{Uniform} \\
\cmidrule(lr){2-3} \cmidrule(lr){4-5} \cmidrule(lr){6-7} \cmidrule(lr){8-9}
 & F1($\uparrow$) & SHD($\downarrow$) & F1($\uparrow$) & SHD($\downarrow$) & F1($\uparrow$) & SHD($\downarrow$)  & F1($\uparrow$) & SHD($\downarrow$) \\

\cmidrule(lr){1-9}
Loss+   & \textbf{0.750} & \textbf{15.2} & \textbf{0.734} & \textbf{16.4} & \textbf{0.732} & \textbf{16.2} & \textbf{0.777} & \textbf{13.8} \\
L2      & 0.661 & 23.6 & 0.595 & 29.4 & 0.633 & 25.9 & 0.668 & 23.6 \\
Loss          & 0.729 & 16.3 & 0.703 & 18.4 & \textbf{0.732} & \textbf{16.3} & 0.759 & 14.3 \\
None          & 0.718 & 19.3 & 0.639 & 25.9 & 0.682 & 22.2 & 0.710 & 20.0 \\
\bottomrule
\end{tabular}
\caption{Comparison for different normalization methods in RoaDs on SF-2 (linear EV case, $n_v = 20$, $n_d=2n_v$, $p_a,p_b,p_c = 0.3,0.3,1$).}
\label{tab:tab15}
\end{table*}

\begin{table*}[htbp]
\centering
\begin{tabular}{l*{8}{c}}
\toprule
\multirow{2}{*}[\multirowoffset]{Method} & \multicolumn{2}{c}{Gauss} & \multicolumn{2}{c}{Exp} & \multicolumn{2}{c}{Gumbel} & \multicolumn{2}{c}{Uniform} \\
\cmidrule(lr){2-3} \cmidrule(lr){4-5} \cmidrule(lr){6-7} \cmidrule(lr){8-9}
 & F1($\uparrow$) & SHD($\downarrow$) & F1($\uparrow$) & SHD($\downarrow$) & F1($\uparrow$) & SHD($\downarrow$)  & F1($\uparrow$) & SHD($\downarrow$) \\

\cmidrule(lr){1-9}
Loss+   & \textbf{0.821} & \textbf{11.4} & \textbf{0.777} & \textbf{14.6} & \textbf{0.805} & \textbf{12.7} & 0.818 & 11.6 \\
L2      & 0.689 & 22.0 & 0.630 & 26.3 & 0.699 & 21.2 & 0.689 & 21.7 \\
Loss          & 0.815 & 11.6 & 0.766 & 15.5 & 0.796 & 13.1 & \textbf{0.828} & \textbf{11.1} \\
None          & 0.718 & 19.1 & 0.675 & 23.4 & 0.741 & 18.1 & 0.714 & 19.6 \\
\bottomrule
\end{tabular}
\caption{Comparison for different normalization methods in RoaDs on ER-2 (linear EV case, $n_v = 20$, $n_d=2n_v$, $p_a,p_b,p_c = 0.3,0.3,1$).}
\label{tab:tab16}
\end{table*}

\begin{table}[htbp]
\centering
\begin{tabular}{l*{4}{c}}
\toprule
\multirow{2}{*}[\multirowoffset]{Method} & \multicolumn{2}{c}{Gauss} & \multicolumn{2}{c}{Exp} \\
\cmidrule(lr){2-3} \cmidrule(lr){4-5} 
 & F1($\uparrow$) & SHD($\downarrow$) & F1($\uparrow$) & SHD($\downarrow$) \\
\cmidrule(lr){1-5}
Linear & \textbf{0.821} & \textbf{11.4} & \textbf{0.777} & \textbf{14.6}  \\
Lasso     & 0.818 & 11.5 & 0.760 & 15.6 \\
\bottomrule
\end{tabular}
\caption{Comparison for different surrogate models in RoaDs on ER-2 (linear EV case, $n_v = 20$, $n_d=2n_v$, $p_a,p_b,p_c = 0.3,0.3,1$).}
\label{tab:tab17}
\end{table}

\begin{table}[htbp]
\centering
\begin{tabular}{l*{4}{c}}
\toprule
\multirow{2}{*}[\multirowoffset]{Method} & \multicolumn{2}{c}{MLP} & \multicolumn{2}{c}{GP} \\
\cmidrule(lr){2-3} \cmidrule(lr){4-5} 
 & F1($\uparrow$) & SHD($\downarrow$) & F1($\uparrow$) & SHD($\downarrow$) \\
\cmidrule(lr){1-5}
Radom Forest & \textbf{0.578} & \textbf{25.9} & \textbf{0.350} & \textbf{32.6}  \\
Polynomial     & 0.553 & 27.9 & \textbf{0.350} & 33.2 \\
\bottomrule
\end{tabular}
\caption{Comparison for different surrogate models in RoaDs on ER-2 (nonlinear case, $n_v = 20$, $n_d=2n_v$, $p_a,p_b,p_c = 0.3,0.3,1$).}
\label{tab:tab18}
\end{table}

\begin{table}[htbp]
\centering
\footnotesize
\begin{tabular}{l*{8}{c}}
\toprule
\multirow{2}{*}[\multirowoffset]{Method} & \multicolumn{2}{c}{thres = 0.05} & \multicolumn{2}{c}{thres = 0.2} & \multicolumn{2}{c}{thres = 0.3} \\
\cmidrule(lr){2-3} \cmidrule(lr){4-5} \cmidrule(lr){6-7}
 & F1& SHD & F1 & SHD & F1 & SHD  \\
 \cmidrule(lr){1-7}
PC\_stable & 0.333 & 14.0  & 0.333 & 14.0 & 0.333 & 14.0 \\
LiNGAM & - & - & - &- & - & - \\
NTS-B   & 0.308 & 15.0  & 0.333 &14.0 &0.364& \textbf{13.0} \\
ECA   &  0.414 & 17.0 &0.414 & 17.0 & \textbf{0.414} & 17.0  \\
RoaDs  & \textbf{0.563} & \textbf{13.0} & \textbf{0.417} & \textbf{13.0} &0.364&\textbf{13.0} \\
\cmidrule(lr){1-7}
GOLEM-NV & 0.364 & 13.0 & 0.364 & 13.0& 0.364 & 13.0 \\
\bottomrule
\end{tabular}
\caption{Comparison on Saches dataset with different thresholds.}
\label{tab:tab19}
\end{table}

\begin{figure}[htbp]
    \centering
    \includegraphics[width=0.9\linewidth]{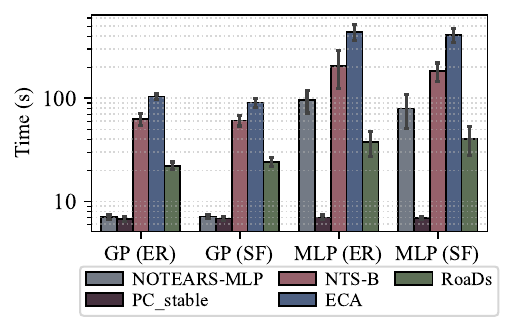}
    \caption{Time cost of continuous-based methods and PC-stable (nonlinear SEM, \(n_v=20,n_d=2n_v,k=2\)).}
    \label{fig:fig14}
\end{figure}

\begin{figure*}[htbp]
    \centering
    \includegraphics[width=1.0\linewidth]{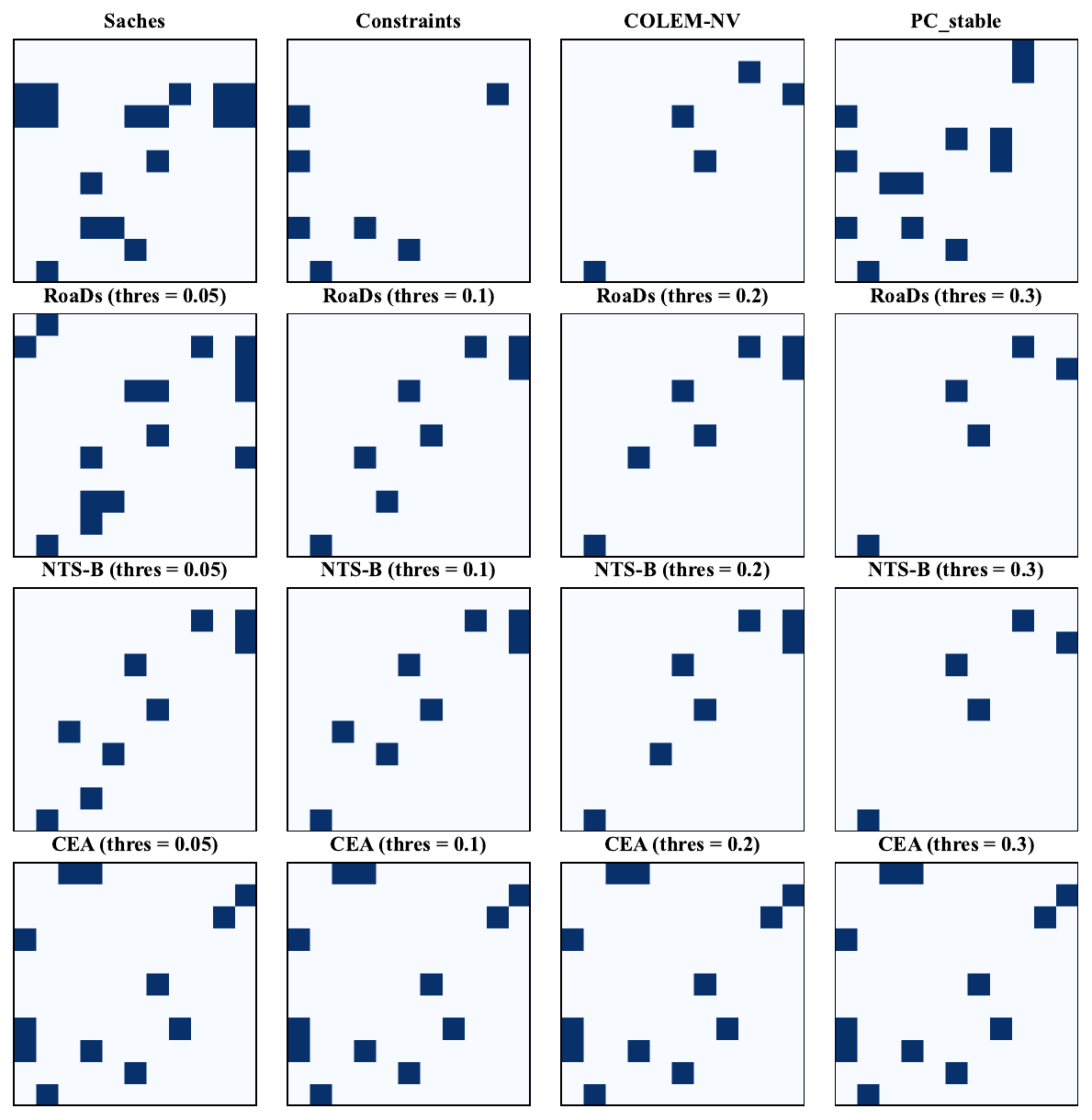}
    \caption{Visualization of DAG learned by different methods on Saches \cite{sachs2005causal}. }
    \label{fig:fig15}
\end{figure*}

\clearpage

\bibliography{allrefs}


\end{document}